%% file: main.tex
\documentclass{article}
\input{tex/preamble}

\title{A Step in the Right Place:
Understanding and Optimizing Timestep Placement in Continuous Diffusion Language Models}

\author{
  Shutian Zheng$^{1,*}$ \quad
  Zhihong Yu$^{3,*}$ \quad
  Linye Wei$^{2,4}$ \quad
  Zeyang Zhang$^{5}$ \quad
  Meng Li$^{2,4,6,\dagger}$
  \\
  $^{1}$College of Engineering, Peking University
  \\
  $^{2}$Institute for Artificial Intelligence, Peking University
  \\
  $^{3}$Weiyang College, Tsinghua University
  \\
  $^{4}$School of Integrated Circuits, Peking University
  \\
  $^{5}$Institute for Interdisciplinary Information Sciences, Tsinghua University
  \\
  $^{6}$Beijing Advanced Innovation Center for Integrated Circuits
}

\hypersetup{
  colorlinks=true,
  citecolor=paperlink,
  linkcolor=paperlink,
  urlcolor=paperlink,
  pdftitle={A Step in the Right Place: Understanding and Optimizing Timestep Placement in Continuous Diffusion Language Models},
  pdfauthor={Shutian Zheng, Zhihong Yu, Linye Wei, Zeyang Zhang, Meng Li}
}
\date{}

\begin{document}
  \maketitle
  \fancyhead[L]{\small Preprint}
  \fancyhead[R]{}
  \renewcommand{\headrulewidth}{0.4pt}

  \begingroup
    \renewcommand{\thefootnote}{}
    \begin{NoHyper}
      \footnotetext{
        \textsuperscript{*}Equal contribution.
        \quad
        \textsuperscript{$\dagger$}Corresponding author: \texttt{meng.li@pku.edu.cn}.
      }
    \end{NoHyper}
    \addtocounter{footnote}{-1}
  \endgroup

  \input{sections/abstract}
  \input{sections/introduction}
  \input{sections/related_work}
  \input{sections/timestep_placement}
  \input{sections/method}

  \FloatBarrier
  \setcounter{topnumber}{2}
  \setcounter{totalnumber}{3}
  \input{sections/experiments}
  \input{sections/conclusion}
  \input{sections/ai_use_statement}
  \input{sections/ethics_statement}
  \input{sections/reproducibility_statement}

  \bibliography{references}
  \bibliographystyle{steptune_references}

  \clearpage
  \restoreappendixlayout
  \input{sections/appendix}
\end{document}

%% file: tex/preamble.tex
\usepackage{steptune_preprint,times}
\usepackage{amsmath,amssymb,mathtools,amsthm}
\usepackage{graphicx}
\usepackage{xcolor}
\usepackage{microtype}
\usepackage{booktabs,multirow}
\usepackage{enumitem}
\usepackage{wrapfig}
\usepackage{ragged2e}
\usepackage{caption}
\makeatletter
\let\steptuneoriginalfloatplacement\@addtocurcol
\newcommand{\restoreappendixlayout}{%
  \let\@addtocurcol\steptuneoriginalfloatplacement
  \setlength{\columnsep}{0pt}%
  \setlength{\intextsep}{0pt}%
  \captionsetup{skip=10pt}%
}
\makeatother
\usepackage{flafter}
\usepackage{placeins}
\usepackage{hyperref}
\usepackage{xurl}
\definecolor{paperlink}{rgb}{0,0.08,0.45}
\input{tex/macros}

%% file: tex/macros.tex
\newcommand{\method}{StepTune}
\newsavebox{\stepLeftTable}
\newsavebox{\stepRightTable}
\newtheorem{theorem}{Theorem}
\newtheorem{proposition}[theorem]{Proposition}

%% file: sections/abstract.tex
\begin{abstract}
  Continuous diffusion language models represent an emerging
  paradigm for language modeling. They generate text in parallel by iteratively
  refining continuous representations and mapping them back to tokens through
  a decoder. However, realizing the efficiency advantage of parallel generation
  requires maintaining generation quality with only a few denoising steps.
  Even when the number of steps is fixed, their placement along the denoising
  trajectory remains flexible. In this paper, we demonstrate that changing timestep placement alone can
  substantially alter the final latent representations and thereby affect the
  decoded token distributions. Building on these findings, we introduce \textbf{\method{}}, an algorithm
  that searches for timestep placements that improve generation quality. The search is conducted on a small calibration set and guided
  by a language-aware objective, with the model and sampling budget fixed. Once calibrated, the timestep placement is fixed and reused
  across inputs.
  Experiments demonstrate consistent improvements in unconditional generation
  perplexity and conditional task metrics across all evaluated sampling budgets.
  In particular, StepTune enables 16-step sampling with a frozen model to
  outperform its native 64-step baseline on both XSum summarization and WMT14
  translation and reduces end-to-end latency by approximately $3.8\times$.
  These results establish timestep placement as a practical means of improving
  the quality--efficiency trade-off in continuous diffusion language models.

\end{abstract}

%% file: sections/introduction.tex
\section{Introduction}
\label{sec:introduction}

\begingroup
\setlength{\columnsep}{6pt}
\setlength{\intextsep}{3pt}

\begin{wrapfigure}{r}{0.44\textwidth}
  \vspace{-0.5\baselineskip}
  \centering

  \includegraphics[
  width=\linewidth,
  keepaspectratio
  ]{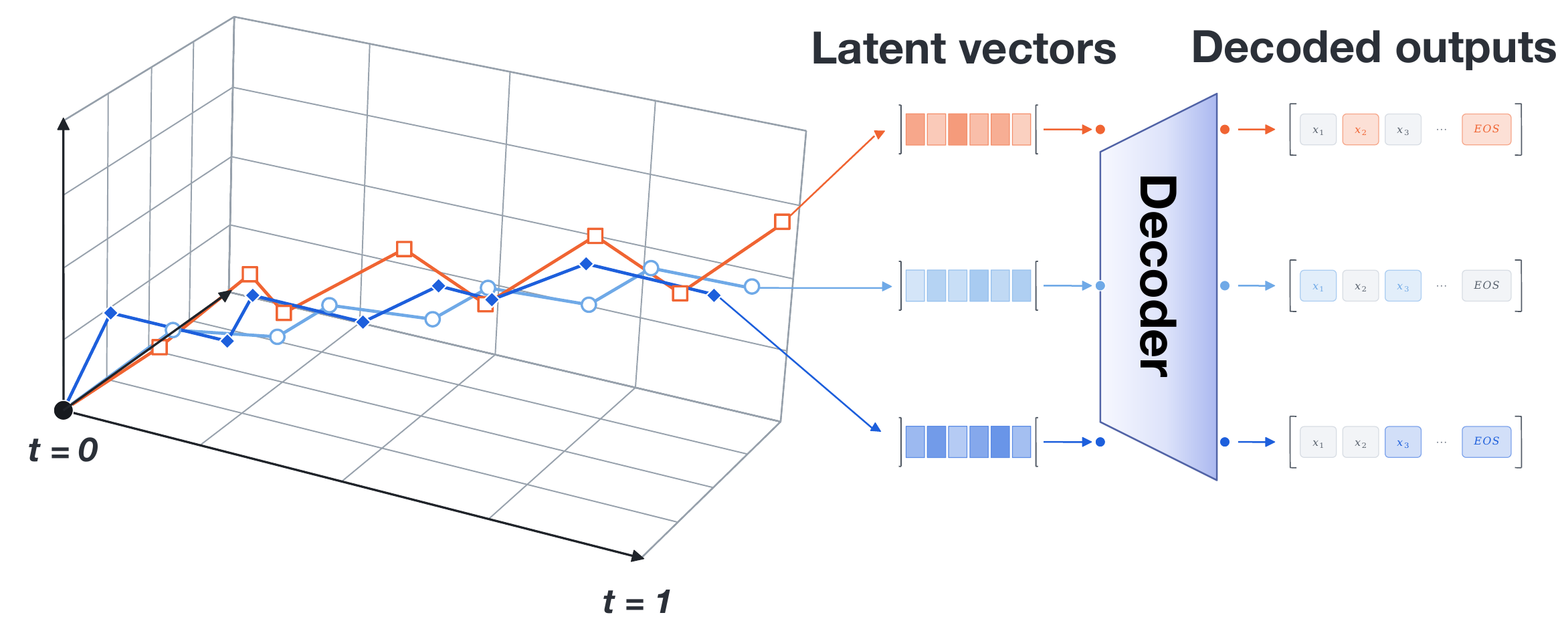}

  \captionsetup{skip=3pt}
  \caption{\textbf{Timestep placement shapes generation.}
  Different timestep placements alter the sampling trajectory,
  eventually leading to different decoded outputs.}
  \label{fig:intro-gridcal}
\end{wrapfigure}

Autoregressive language models generate text one token at a time,
conditioned on the preceding sequence \citep{radford2019language}.
Diffusion language models offer an alternative by iteratively
denoising a sequence, allowing multiple positions to be updated
in parallel \citep{li2022diffusion,sahoo2024simple}.
Within this paradigm, continuous diffusion language models
(continuous DLMs), such as ELF \citep{hu2026elf} and
AURORA \citep{liang2026aurora}, map tokens into continuous
representations, perform denoising in this space, and decode
the resulting representations back into tokens.
However, generating a sequence in parallel still
requires sequential rounds of refinement. Realizing the efficiency advantage of parallel generation
therefore requires maintaining generation quality with only
a few denoising steps, measured by the number of function
evaluations (NFE).
\par
\endgroup

\begingroup
\setlength{\textfloatsep}{12pt}

\begin{figure}[t]
  \centering

  \includegraphics[width=\textwidth]{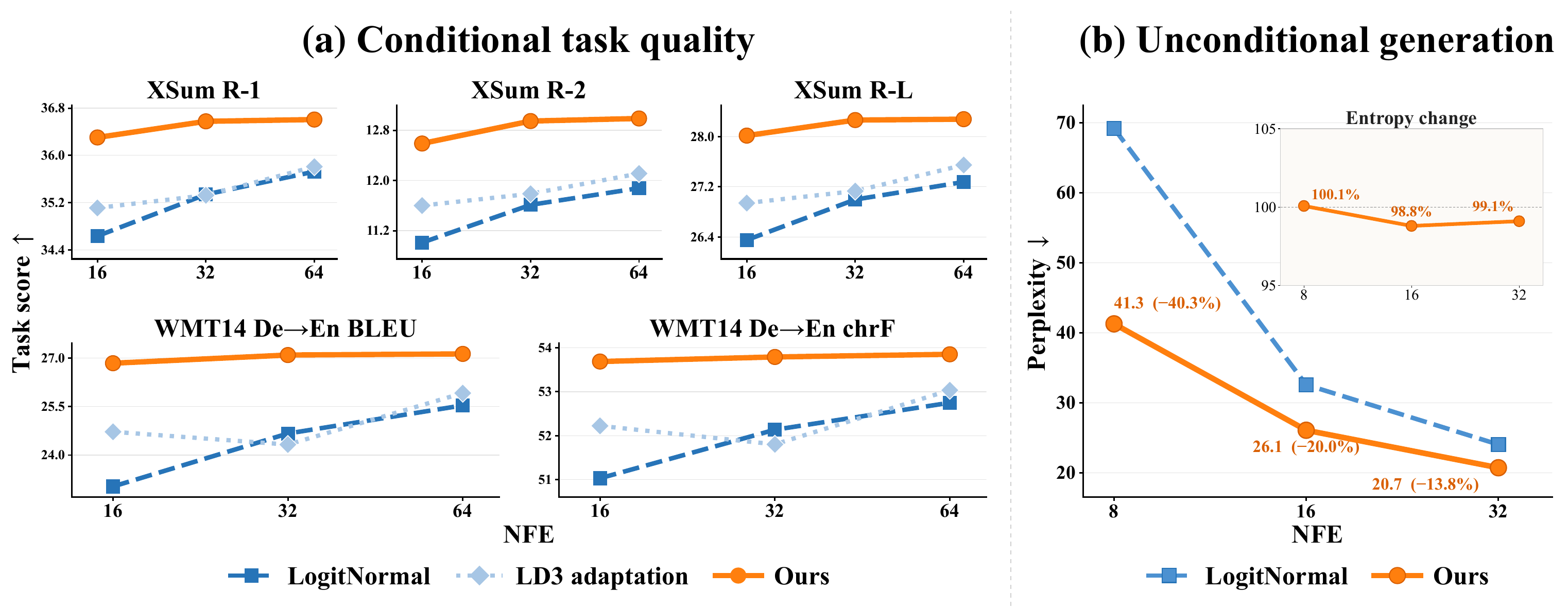}

  \caption{
  \textbf{StepTune improves generation quality with fewer denoising steps.}
  (a) On XSum summarization and WMT14 De$\rightarrow$En translation,
  16-step sampling outperforms the native 64-step baseline.
  (b) Unconditional generation with the official ELF SDE sampler
  shows consistently lower perplexity at matched NFE.
  The inset reports unigram entropy relative to LogitNormal
  ($100\%$).
  }
  \label{fig:main-generation}

\end{figure}

At a limited NFE budget, the placement of denoising steps
along the denoising trajectory remains flexible.
With the model and solver unchanged,  different timestep
placements
allocate computation differently between the initial noise and the final
representation. Prior work on image diffusion has demonstrated the benefits of
optimizing this allocation \citep{sabour2024align,tong2025ld3}.
For language generation, we seek to determine whether timestep
placement offers similar benefits for token distributions
and generation quality.
This raises two central questions:
\textit{(i) How much control does timestep placement provide
over the final token distributions?
(ii) How can we use this control to improve generation
at a fixed NFE budget?}

To address question~(i), we empirically investigate the effects
of timestep placement in Section~3.
We first demonstrate that changing timestep placement alone
can substantially alter the final latent representations,
which the decoder subsequently maps to token distributions
(Figure~\ref{fig:intro-gridcal}).
To better assess differences in token distributions, we use
a high-NFE trajectory from the same model as a numerical reference.
It uses the same input and initial noise but takes finer
denoising steps to obtain a more accurate numerical solution.

Using this reference, our analysis in
Section~\ref{sec:decoder-sensitivity} reveals two distinct regimes.
When latent deviations from the reference are large, latent
error is strongly positively correlated with the KL divergence
between decoded token distributions: smaller latent errors
tend to correspond to smaller KL.
As the latent representations
approach the reference, the correlation between latent error
and KL drops sharply. Because the decoder's sensitivity to error direction becomes
important at small scales, latent distance alone no longer
adequately captures deviations in decoded token distributions.

These findings provide a bridge from question~(i)
to question~(ii).
Finding improved timestep placements is a search problem
with infinitely many possibilities in a continuous space,
so an objective is needed to guide the search.
Euclidean and perceptual distances to reference outputs
have been used as optimization objectives in image diffusion
sampling \citep{shaul2024bespoke,tong2025ld3,frankel2025s4s}.
For language generation, however, image-specific perceptual
metrics such as LPIPS \citep{zhang2018unreasonable}
are not directly applicable.
Our analysis also shows that latent distance alone cannot
adequately capture changes in token distributions.

We therefore draw on the complementary behavior observed
above to construct a joint discrepancy.
Intuitively, the latent term encourages movement toward
the reference, providing guidance at larger deviations
where latent error and KL are strongly correlated.
Closer to the reference, the KL term distinguishes changes
in token probabilities that latent distance alone cannot capture.
Together, these terms favor representations that are close
to the reference and yield similar token distributions.
For timestep placement $\theta$, let $A(\theta)$ denote
latent error and $K(\theta)$ denote decoder KL.
With positive weights $w_A$ and $w_K$, we define
\begin{equation}
  F(\theta)
  =
  w_A A(\theta) + w_K K(\theta).
  \label{eq:intro-objective}
\end{equation}

Building on this objective, we now turn to question~(ii)
and introduce \textbf{\method{}}, an algorithm that searches
for improved timestep placements on a small calibration set
with the model, solver, and NFE fixed.
StepTune proceeds in three stages
(Section~\ref{sec:method}).
First, we compress the space of timestep placements through
a compact parameterization of step allocation or a finite
set of candidate placements to make the search tractable.
Second, we use $F$ to guide the search toward regions
with the potential to improve generation quality,
subject to constraints evaluated on the calibration set.
Third, we select the final timestep placement based on
calibration results and freeze it for reuse.

By changing only timestep placement, StepTune delivers consistent improvements in generation quality. On ELF \citep{hu2026elf}, it improves unconditional generation perplexity and conditional task metrics across all evaluated NFE budgets (Figure~\ref{fig:main-generation}). In particular, 16-step sampling outperforms the native 64-step baseline on both XSum summarization \citep{Narayan2018DontGM} and WMT14 translation \citep{wmt14}, achieving an approximately $3.8\times$ end-to-end speedup with the same frozen model. These gains are not limited to a single continuous DLM architecture. On AURORA \citep{liang2026aurora}, StepTune also enables 16-step sampling to outperform the native 32-step baseline on XSum summarization. For unconditional generation, it reduces perplexity by 26.9\% at the same 16-step budget. Together, these results demonstrate that optimizing timestep placement can improve both generation quality and sampling efficiency across distinct continuous DLM architectures.

Our main contributions are threefold.
\begin{enumerate}[leftmargin=*]
  \item \textbf{Understanding timestep placement.}
  We show how timestep placement alters terminal representations and why latent distance alone can fail to capture changes in decoded token distributions.

  \item \textbf{Language-aware calibration.}
  We introduce StepTune to optimize timestep placement using $F$ under quality constraints and freeze the selected schedule for reuse across inputs.

  \item \textbf{Improved quality and efficiency.}
  We demonstrate gains across continuous DLM architectures, achieving better generation quality with fewer denoising steps and no model updates.
\end{enumerate}

\par
\endgroup

%% file: sections/related_work.tex
\section{Related Work}

\paragraph{Diffusion language models.}
Diffusion language models generate text through iterative denoising.
Discrete DLMs define corruption and denoising directly in token space
\citep{austin2021structured,sahoo2024simple,nie2025llada},
whereas continuous DLMs operate on continuous token representations
and decode the resulting representations into text
\citep{li2022diffusion,dieleman2022continuous,
gulrajani2023likelihood,hu2026elf}.
More closely related work studies training noise schedules
and inference timestep placement.
LangFlow learns a Gumbel noise schedule motivated by
information-uniform denoising \citep{chen2026langflow},
RePlaid learns its noise schedule by reducing ELBO-estimator
variance \citep{yang2026replaid}, and ELF demonstrates
the benefits of logit-normal over uniform inference times
\citep{hu2026elf}.
Whereas these studies learn training schedules or compare predefined inference rules, we study inference timestep placement as the sole variable. We optimize it directly, with all other settings fixed, and our method has no hard dependency on any predefined schedules.

\paragraph{Optimizing timestep placement.}
Prior work has directly optimized inference timestep placement
in visual diffusion.
AYS derives its objective from a bound on sampling-distribution
divergence \citep{sabour2024align}, while GITS uses dynamic
programming to select timesteps based on trajectory structure
\citep{chen2024trajectory}.
LD3 learns timesteps by matching the outputs of a low-NFE
sampler to a high-NFE reference \citep{tong2025ld3}.
S4S-Alt also uses reference matching but jointly optimizes
timesteps and solver coefficients \citep{frankel2025s4s}.
OYS instead uses Bayesian optimization to search directly
against image-quality metrics \citep{zhang2026oys}.
In contrast, we study timestep placement in continuous DLMs.
We investigate the language-specific relationship between
changes in terminal representations and the resulting changes
in decoded token distributions.
Building on this analysis, we develop a joint objective
and a calibration algorithm to improve language generation.
We provide an extended discussion of additional related work
in Appendix~\ref{app:related-work}.

%% file: sections/timestep_placement.tex
\section{Understanding Timestep Placement}
\label{sec:theory}

In this section, we seek to understand the role of timestep
placement in continuous DLM generation.
We empirically examine how it changes terminal latent
representations and how these changes propagate through
the decoder to affect token distributions. We further test whether the effects of a calibrated timestep
placement persist across held-out inputs.

\paragraph{Preliminaries.}
Continuous DLMs use an encoder to represent text in a continuous
latent space and a decoder to map latent representations back
to token probabilities \citep{hu2026elf,liang2026aurora}.
Let $t$ denote progress along the sampling process during inference.
Sampling starts from Gaussian noise at $t=0$ and progressively
transforms it into a latent representation at $t=1$ through finite denoising steps.
At step $k$, let $z_k$ and $z_{k+1}$ denote the latent states
before and after an update from time $t_k$ to $t_{k+1}$.
Let $v_k$ denote the velocity computed from the model's prediction
at the current state and time.
A typical Euler update takes the form
\begin{equation}
  z_{k+1}=z_k+(t_{k+1}-t_k)v_k.
  \label{eq:sampling_update}
\end{equation}

Each Euler update treats the predicted velocity as constant
over the time interval.
Multiplying this velocity by the interval length approximates
the change in the latent state.
Sampling thus approximates the continuous denoising process
through discrete updates over a finite sequence of time intervals.
Timestep placement determines how these intervals are distributed
along the sampling process.
We denote the terminal latent representation by $Z$.
The decoder converts $Z$ into token probabilities from which
the output text is generated. Appendix~\ref{app:theory} provides a theoretical account of
the full sampling and decoding process, while the following
subsections present empirical and illustrative analyses.

\subsection{Timestep Placement Shapes Terminal Representations}
\label{sec:terminal-displacement}

Intuitively, different timestep placements can lead to
different terminal representations even at the same NFE.
But how large can this effect become in a continuous DLM?

\setlength{\columnsep}{10pt}%
\begin{wrapfigure}{R}{0.4\textwidth}

  \centering
  \includegraphics[width=\linewidth]{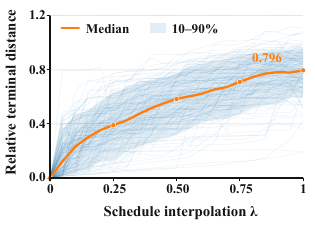}

  \caption{
  \textbf{Timestep placement changes terminal representations.}
  Relative displacement from uniform sampling.
  }

  \label{fig:schedule-response}
\end{wrapfigure}%

We test this effect using ELF-B on 128 held-out XSum inputs
at NFE--16.
We gradually shift the timesteps from a uniform placement
($\lambda=0$) to a fixed nonuniform placement ($\lambda=1$),
keeping each input's initial noise unchanged.
Let $Z_\lambda$ denote the resulting terminal representation
and $Z_U$ the uniform baseline's terminal representation.
We measure their relative terminal distance as
$\|Z_\lambda-Z_U\|_2/\|Z_U\|_2$ over the generated positions.

Figure~\ref{fig:schedule-response} shows that, at $\lambda=1$,
the median terminal displacement reaches 79.6\% of the
uniform endpoint's norm.
Even at $\lambda=0.1$, moving the timesteps only one tenth
of the way toward the nonuniform placement already produces
a median displacement of 23.3\%.
\textbf{Timestep placement alone can substantially change
the final latent representations at a fixed NFE budget.}

\subsection{From Latent Error to Token Divergence}
\label{sec:decoder-sensitivity}

\setlength{\columnsep}{10pt}%
\begin{wrapfigure}{R}{0.4\textwidth}

  \centering
  \includegraphics[width=\linewidth]{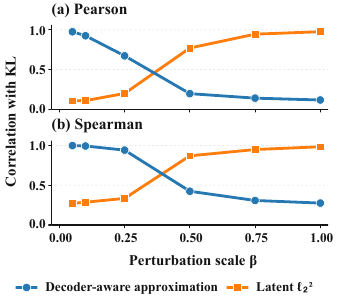}

  \caption{\textbf{Correlations with KL across six perturbation scales.}}
  \label{fig:mechanism-decoder}
\end{wrapfigure}%

\input{figures/calibration_generalization_float}
\input{figures/steptune_pipeline_float}

We next examine how terminal latent shifts affect decoded
token distributions.
Since finer sampling steps more accurately approximate
the continuous generation process, we use a high-NFE run
as a reference under the same model, solver, input,
and initial noise.
Let $Z^\star$ denote the reference endpoint, $Z$ a sampled
endpoint, and $q(Z)$ the decoded token probabilities.
The latent error is $e=Z-Z^\star$.
A matrix $W_{\mathrm{dec}}$ describes the decoder's local
sensitivity to different error directions, giving the
\emph{decoder-aware quadratic approximation}
\begin{equation}
  \mathrm{KL}\!\left(q(Z^\star)\,\|\,q(Z^\star+e)\right)
  \approx \frac{1}{2}e^\top W_{\mathrm{dec}}e.
  \label{eq:decoder-sensitivity}
\end{equation}
This expression depends on both error magnitude and direction.
The derivation is provided in Appendix~\ref{app:theory}.

We test this relationship using 96 terminal-error directions
from NFE--16 schedules on XSum, with an NFE--64 reference.
Along each direction, we evaluate the decoder at
$Z^\star+\beta e$ for six scales from $\beta=0.05$ to $1$.
At each scale, we compare how strongly normalized squared
latent distance and the quadratic approximation correlate
with the actual KL divergence.

Figure~\ref{fig:mechanism-decoder} reveals two distinct regimes.
At the largest scale, latent distance is strongly positively
correlated with KL: smaller latent errors tend to accompany
smaller deviations in token distributions.
Near the reference, however, this relationship weakens sharply.
At $\beta=0.05$ and $0.1$, the Pearson correlation between
latent distance and KL drops to about $0.11$, whereas the
decoder-aware approximation achieves $0.975$ and $0.925$.
Spearman correlations show the same pattern.

The quadratic approximation helps explain why latent distance
alone is insufficient near the reference: latent distance
measures error magnitude but ignores direction, whereas
the approximation accounts for the decoder's sensitivity
to different error directions.
Its weaker correlation with KL at larger scales is expected,
since it describes local behavior around the reference.

\textbf{Latent distance is strongly correlated with KL
at larger deviations, but near the reference, latent
distance alone cannot adequately capture deviations
in decoded token distributions.}

\subsection{A Joint Objective for Timestep Calibration}
\label{sec:stable-reduction}

The preceding analysis examines how timestep placement affects
terminal representations and decoded token distributions.
Our goal, however, is to find timestep placements that improve
generation quality.
Finding such placements in a continuous space is a search
problem, which our algorithm in Section~\ref{sec:method} addresses.
To guide this search under a limited calibration budget,
we need an objective that provides useful guidance and
whose reductions persist beyond the calibration trajectories.

Our primary interest is agreement with the reference token
distributions, measured by decoder KL.
Our preceding analysis shows that, at larger deviations,
smaller latent errors tend to correspond to smaller KL.
Near the reference, however, latent distance alone cannot
adequately capture differences in token probabilities
because of the decoder's sensitivity to error direction.

We therefore formulate a joint objective that accounts
for both latent error and decoder KL.
Latent error provides a direct signal of endpoint deviation
before it is transformed by the decoder, offering auxiliary
guidance toward the reference, particularly at larger deviations.
Decoder KL captures the actual differences in token
probabilities, including those that latent distance alone
misses near the reference. Section~\ref{sec:objective-ablation} compares this objective with
using latent error or decoder KL alone.

For a schedule $\theta$, let $A(\theta)$ and $K(\theta)$ denote
the mean relative squared latent error and mean decoder KL,
respectively, over the evaluation trajectories.
Using positive baseline normalization constants
$A_{\mathrm{scale}}$ and $K_{\mathrm{scale}}$
and a weight $\alpha\in[0,1]$, we define
\begin{equation}
  F_\alpha(\theta)
  =
  \alpha\frac{A(\theta)}{A_{\mathrm{scale}}}
  +
  (1-\alpha)\frac{K(\theta)}{K_{\mathrm{scale}}}.
  \label{eq:joint-discrepancy}
\end{equation}
We set $\alpha=1/2$; normalization details are provided
in Appendix~\ref{app:experimental-details}.

We test whether calibration with this objective substantially
reduces both latent error and, more importantly, decoder KL
on unseen trajectories across model scales and NFE budgets.
We calibrate one schedule for NFE--16 and another for NFE--32
using 16 trajectories from ELF-M.
Both schedules are then frozen and evaluated on ELF-B,
ELF-M, and ELF-L, using 160 independent validation trajectories
for each model and budget.

Figure~\ref{fig:schedule-control} shows that the joint
discrepancy decreases by approximately 68--80\% relative
to LogitNormal across the six validation settings.
Crucially, both components decrease in every setting:
latent error falls by 59.8--81.5\%, and decoder KL
by 71.6--78.2\%.
The improvement thus reflects reductions in both forms
of deviation, rather than a decrease in one component
offsetting an increase in the other.

\textbf{Calibrating timestep placement with the joint
objective substantially reduces both latent error and
decoder KL, with these benefits persisting on unseen
trajectories and across model scales.}

%% file: figures/calibration_generalization_float.tex
\begin{figure}[t]
  \centering
  \includegraphics[
  width=\textwidth,
  height=0.40\textheight,
  keepaspectratio,
  trim=0 1cm 0 0,
  clip
  ]{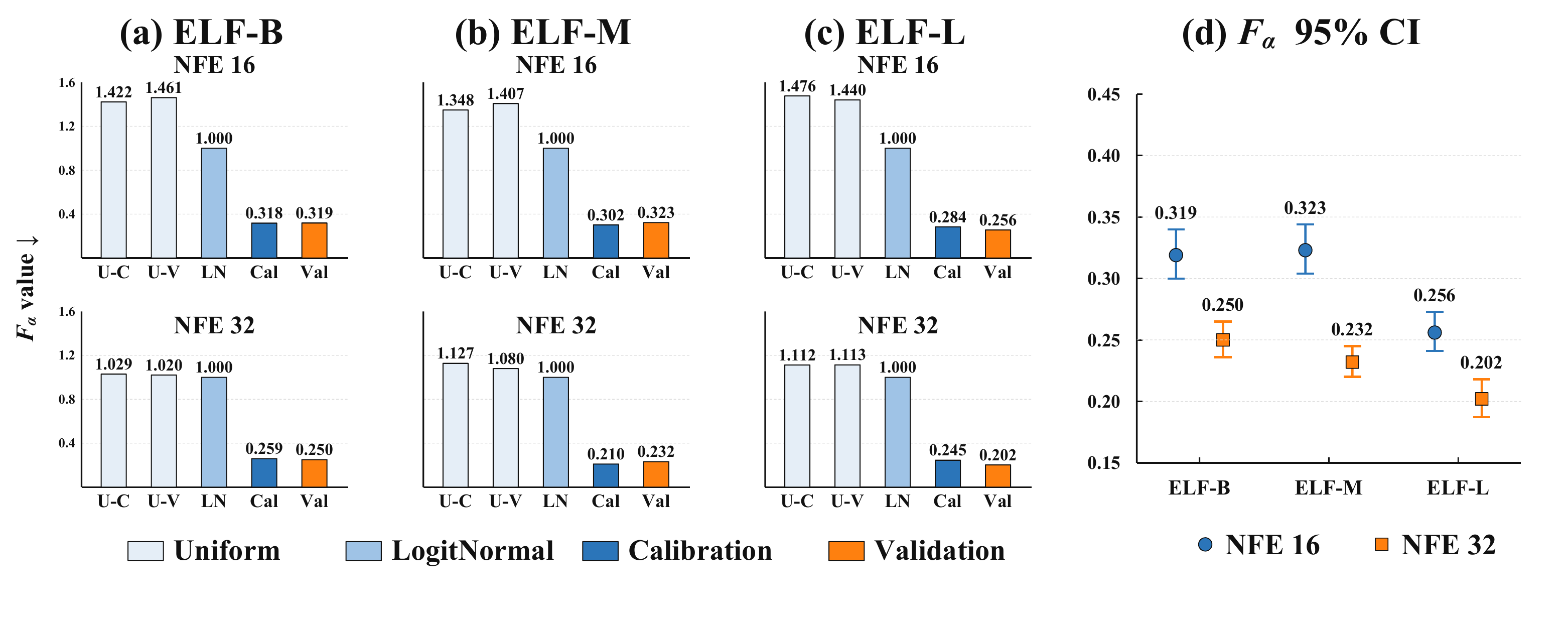}

  \caption{\textbf{Joint discrepancy reductions persist beyond calibration.}
  (a--c) Frozen schedules reduce normalized $F_{\alpha}$
  on calibration and unseen validation trajectories
  (LogitNormal $=1$).
  (d) Validation scores with 95\% confidence intervals.}
  \label{fig:schedule-control}
\end{figure}

%% file: figures/steptune_pipeline_float.tex
\begin{figure}[t]
  \centering
  \includegraphics[width=\linewidth]{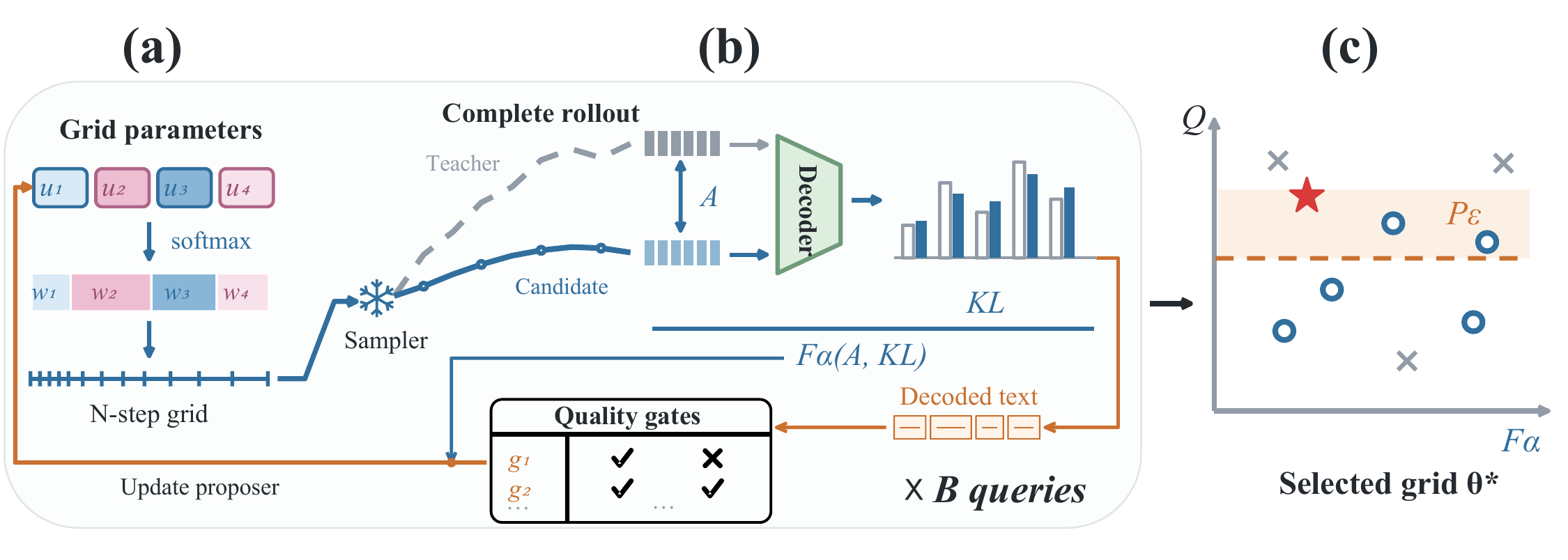}
  \caption{
  \textbf{Overview of \method{}.}
  Compress the schedule space, search using $F_\alpha$
  under quality constraints, and select and freeze
  one schedule for reuse.
  }
  \label{fig:pipeline}
\end{figure}

%% file: sections/method.tex
\section{StepTune: Tuning Timesteps for Continuous DLMs}
\label{sec:method}

In this section, we introduce \textbf{StepTune}, an algorithm
that searches for improved timestep placements under a limited
calibration budget, guided by the objective $F_\alpha$
in Equation~(\ref{eq:joint-discrepancy}). Each timestep placement defines a sampling schedule,
an ordered sequence of time points used during generation.
The algorithm proceeds in three stages (Figure~\ref{fig:pipeline}):
(I) \emph{Compress} the schedule space,
(II) \emph{Search} using $F_\alpha$ under quality constraints,
and (III) \emph{Select and freeze} one schedule for reuse across inputs.

\paragraph{Stage I: Compress the schedule space.}
Let $N$ denote the fixed NFE budget.
Directly optimizing all $N-1$ interior timesteps creates
a high-dimensional search problem.
We therefore compress the schedule space to make the search
tractable under a limited calibration budget.

We divide the sampling interval into $R$ regions with boundaries
$0=b_0<b_1<\cdots<b_R=1$.
Each region $r$ receives $n_r\geq1$ uniformly spaced steps,
with $\sum_{r=1}^{R}n_r=N$.
We search only the interior boundaries and step allocations;
these determine all timesteps while keeping the endpoints
and total NFE fixed.
Section~\ref{sec:query-efficiency} evaluates the query efficiency
of this compressed search space.

Alternatively, we can restrict the search to a finite bank
of schedules generated from chosen placement rules.
Both approaches reduce the search space without changing
the inference budget.

\paragraph{Stage II: Search under quality constraints.}
We use a calibration set of $M$ instances, each consisting
of an input and an initial noise realization.
For each instance, we cache a high-NFE reference generated
with the same input and initial noise.
Each candidate schedule is evaluated on the entire set and scored using \(F_\alpha\), with the corresponding language quality metrics also recorded.

Quality constraints limit degradation in task performance
and guard against undesirable generation behavior,
such as excessive repetition or abnormal output length.
For each requirement $j$, we define a margin $g_j(\theta)$
for schedule $\theta$, with $g_j(\theta)\geq 0$
indicating satisfaction.
The metrics and thresholds used in each experiment
are detailed in Appendix~\ref{app:experimental-details}.

Let $B$ denote the budget of schedule evaluations.
Each evaluation runs one schedule over the entire calibration set.
For the parameterized search, constrained Bayesian optimization
proposes schedules that balance predicted reductions
in $F_\alpha$ with exploration.
Predicted constraint violations penalize or exclude proposals,
discouraging expenditure on schedules likely to violate
the quality requirements.
A finite candidate bank is evaluated directly.
In both cases, complete sampling runs provide the measured
discrepancy and quality margins used for final selection.

\paragraph{Stage III: Select and freeze.}
After search, let $\mathcal H_{\mathrm{feas}}$ contain
the evaluated schedules that satisfy all measured
quality constraints.
If this set is empty, we retain the baseline schedule.
Otherwise, let $Q(\theta)$ denote the observed calibration
quality, oriented so that larger values are better.
We first retain schedules with near-best quality and then
select the one with the lowest $F_\alpha$.

For a scalar quality criterion, let $\varepsilon_Q\geq 0$
be a tolerance specified before search.
The best observed feasible quality is
$Q_{\max}=\max_{\theta\in\mathcal H_{\mathrm{feas}}}Q(\theta)$.
We define the retained set $\mathcal P_\varepsilon$
and selected schedule $\theta^\star$ as
\begin{equation}
  \begin{aligned}
    \mathcal P_\varepsilon
    &=
    \left\{
    \theta\in\mathcal H_{\mathrm{feas}}:
    Q(\theta)\geq Q_{\max}-\varepsilon_Q
    \right\}, \\
    \theta^\star
    &=
    \arg\min_{\theta\in\mathcal P_\varepsilon}
    F_\alpha(\theta).
  \end{aligned}
  \label{eq:schedule-selection}
\end{equation}
For multiple quality metrics, the retained set uses
the metric-wise or joint tolerance tests described in
Appendix~\ref{app:experimental-details}.
Quality therefore determines the admissible set,
and $F_\alpha$ determines the final choice within it. The design of tolerance helps limit overfitting to calibration scores.

We freeze the selected timestep placement before evaluation
on untouched test data and reuse it across inputs.
Deployment changes only timestep placement, requiring
no model updates, per-example search, backward passes,
or additional model evaluations.

%% file: sections/experiments.tex
\section{Experiments}
\label{sec:experiments}

\paragraph{Experimental setup.}
We evaluate StepTune with ELF \citep{hu2026elf} and
AURORA-LM \citep{liang2026aurora} on unconditional and
conditional generation.
For ELF, unconditional experiments use ELF-B at NFE--8/16/32
and report perplexity and unigram entropy.
Conditional experiments use task-specific checkpoints on
XSum and WMT14 De$\rightarrow$En at NFE--16/32/64,
reporting ROUGE-1/2/L and BLEU/chrF.
For AURORA-LM-S, we evaluate unconditional generation and
XSum summarization at NFE--16 using the same metrics,
additionally reporting MAUVE for unconditional generation.

We compare against ELF's native LogitNormal schedule, AURORA-LM's native tan-d schedule, and an adaptation of LD3 \citep{tong2025ld3} for ELF conditional generation.
Ablation and query-efficiency studies use the ELF
WMT14 De$\rightarrow$En checkpoint.
At matched NFE, only timestep placement changes;
the model and other sampling settings remain fixed.
Schedules are calibrated on small, disjoint sets and frozen
before held-out evaluation.
Complete calibration, reference, sampling, evaluation,
and efficiency protocols are provided in
Appendix~\ref{app:experimental-details}.

\subsection{Main Results}
\label{sec:main-generation}

\begin{table}[t]
  \centering
  \small
  \setlength{\tabcolsep}{3pt}
  \renewcommand{\arraystretch}{1.08}

  \caption{
  \textbf{Conditional generation quality and efficiency on ELF-B.}
  Results use the full XSum and WMT14 test sets (11,334 and 3,003 examples).
  Lat.\ denotes latency in ms at batch size 1, and Thr.\ denotes throughput in
  sequences/s at batch size 64; both are measured on an NVIDIA B200.
  All baselines are reproduced.
  LD3 refers to our adaptation of the original method
  (Appendix~\ref{app:exp-conditional}); LD3 and ours use a 64-NFE teacher for
  NFE--16/32 and a 128-NFE teacher for NFE--64.
  }
  \label{tab:conditional-generation}

  \begin{tabular*}{\textwidth}{
    @{\extracolsep{\fill}}cl*{9}{r}@{}
    }
    \toprule
    NFE & Method
    & \multicolumn{5}{c}{XSum}
    & \multicolumn{4}{c}{WMT14 De$\rightarrow$En} \\
    \cmidrule(lr){3-7}
    \cmidrule(lr){8-11}

    & & R-1 $\uparrow$
    & R-2 $\uparrow$
    & R-L $\uparrow$
    & Lat. $\downarrow$
    & Thr. $\uparrow$
    & BLEU $\uparrow$
    & chrF $\uparrow$
    & Lat. $\downarrow$
    & Thr. $\uparrow$ \\
    \midrule

    \multirow{3}{*}{16}
    & Native LogitNormal
    & 34.636 & 11.010 & 26.351 & 142.11 & 25.83
    & 23.0295 & 51.0300 & 141.09 & 190.97 \\

    & LD3 adaptation
    & 35.112 & 11.600 & 26.940 & -- & --
    & 24.7199 & 52.2235 & -- & -- \\

    & \textbf{Ours}
    & \textbf{36.304} & \textbf{12.589} & \textbf{28.015}
    & 142.20 & 25.83
    & \textbf{26.8390} & \textbf{53.6881}
    & 141.07 & 191.00 \\
    \midrule

    \multirow{3}{*}{32}
    & Native LogitNormal
    & 35.340 & 11.611 & 26.996 & 275.67 & 13.51
    & 24.6668 & 52.1351 & 274.66 & 100.96 \\

    & LD3 adaptation
    & 35.326 & 11.788 & 27.131 & -- & --
    & 24.3254 & 51.8025 & -- & -- \\

    & \textbf{Ours}
    & \textbf{36.577} & \textbf{12.947} & \textbf{28.263}
    & 275.75 & 13.51
    & \textbf{27.0917} & \textbf{53.7921}
    & 274.46 & 101.00 \\
    \midrule

    \multirow{3}{*}{64}
    & Native LogitNormal
    & 35.726 & 11.878 & 27.276 & 543.15 & 6.91
    & 25.5353 & 52.7495 & 541.11 & 51.96 \\

    & LD3 adaptation
    & 35.810 & 12.111 & 27.547 & -- & --
    & 25.9150 & 53.0351 & -- & -- \\

    & \textbf{Ours}
    & \textbf{36.607} & \textbf{12.988} & \textbf{28.277}
    & 543.19 & 6.91
    & \textbf{27.1268} & \textbf{53.8529}
    & 541.03 &51.99 \\

    \bottomrule
  \end{tabular*}
\end{table}

\begin{table*}[t]
  \centering
  \small

  \caption{
  \textbf{Generation results on ELF-B and AURORA-LM-S.}
  (a) Unconditional generation on ELF-B,
  averaged over six runs.
  (b) XSum summarization and OpenWebText unconditional
  generation on AURORA-LM-S, averaged over five runs.
  All baselines are reproduced.
  }
  \label{tab:unconditional-generation}

  \begingroup
  \setlength{\tabcolsep}{3pt}
  \renewcommand{\arraystretch}{1.05}

  \sbox{\stepLeftTable}{%
  \begin{tabular*}{0.43\textwidth}{
    @{\extracolsep{\fill}}lccc@{}
    }
    \toprule
    NFE & 8 & 16 & 32 \\
    \midrule

    \multicolumn{4}{@{}l}{
    \emph{PPL} $\downarrow$
    } \\

    LogitNormal
    & $69.169$ & $32.563$ & $24.021$ \\

    \textbf{Ours}
    & $\mathbf{41.260}$
    & $\mathbf{26.059}$
    & $\mathbf{20.702}$ \\

    \midrule

    \multicolumn{4}{@{}l}{
    \emph{Entropy}
    } \\

    LogitNormal
    & $5.229$ & $5.173$ & $5.157$ \\

    \textbf{Ours}
    & $5.233$ & $5.112$ & $5.111$ \\

    \bottomrule
  \end{tabular*}%
  }

  \sbox{\stepRightTable}{%
  \begin{tabular*}{0.55\textwidth}{
    @{}l@{\hspace{-14pt}\extracolsep{\fill}}cccc@{}
    }
    \toprule
    Method & NFE & \multicolumn{3}{c@{}}{Metrics} \\
    \midrule

    \emph{XSum} &
    & R-1 $\uparrow$
    & R-2 $\uparrow$
    & R-L $\uparrow$ \\

    Native tan-$d$ & 32
    & $36.580$ & $13.479$ & $28.873$ \\

    \textbf{Ours} & 16
    & $\mathbf{36.752}$
    & $\mathbf{13.635}$
    & $\mathbf{29.081}$ \\

    \midrule

    \emph{OpenWebText} &
    & PPL $\downarrow$
    & Entropy
    & MAUVE $\uparrow$ \\

    Native tan-$d$ & 16
    & $56.394$ & $5.386$ & $0.9132$ \\

    \textbf{Ours} & 16
    & $\mathbf{41.210}$
    & $5.356$
    & $\mathbf{0.9326}$ \\

    \bottomrule
  \end{tabular*}%
  }

  \noindent\hbox to \textwidth{%
  \usebox{\stepLeftTable}%
  \hfil
  \vrule width 0.35pt
  height \ht\stepLeftTable
  depth  \dp\stepLeftTable
  \kern 1.5pt
  \vrule width 0.35pt
  height \ht\stepLeftTable
  depth  \dp\stepLeftTable
  \hfil
  \usebox{\stepRightTable}%
  }

  \vspace{3pt}

  \noindent\hbox to \textwidth{%
  \makebox[0.43\textwidth][c]{\textbf{(a) ELF-B}}%
  \hfil
  \makebox[0.55\textwidth][c]{\textbf{(b) AURORA-LM-S}}%
  }

  \endgroup
\end{table*}

\paragraph{Performance on conditional generation.}
StepTune consistently improves conditional generation quality
across tasks and sampling budgets.
On ELF, it outperforms both native LogitNormal and the LD3
adaptation on every reported metric at NFE--16/32/64
(Table~\ref{tab:conditional-generation}).
At NFE--16, ROUGE-2 on XSum increases from 11.010 to 12.589,
while BLEU on WMT14 increases from 23.030 to 26.839
relative to native sampling.
By contrast, the LD3 adaptation improves XSum scores but
falls below native sampling on both WMT14 metrics at NFE--32,
despite using the same teacher budget.

More importantly, our NFE--16 schedules surpass native
NFE--64 on all three ROUGE metrics and both BLEU and chrF.
This quality advantage translates into measured efficiency:
on XSum and WMT14, respectively, generation is
$3.82\times$ and $3.84\times$ faster at batch size~1,
with $3.74\times$ and $3.68\times$ the throughput
at batch size~64.
At matched NFE--16 and 32, latency and throughput remain
essentially unchanged.
Thus, timestep placement improves the quality--efficiency
trade-off by enabling better generation with only one quarter
of the native baseline's evaluations on both tasks.

The benefit also extends to AURORA-LM-S
(Table~\ref{tab:unconditional-generation}b).
On XSum, StepTune at NFE--16 exceeds native tan-$d$
sampling at NFE--32 on all three ROUGE metrics,
achieving 36.752/13.635/29.081 versus
36.580/13.479/28.873.
Improved placement therefore yields higher quality with
half the evaluations in a second continuous DLM architecture,
despite differences in the latent representation, decoder,
and native timestep schedule.

\paragraph{Performance on unconditional generation.}
The improvements also hold for unconditional generation.
On ELF-B, StepTune reduces perplexity relative to native
LogitNormal by 40.3\%, 20.0\%, and 13.8\% at NFE~8, 16,
and 32, respectively
(Table~\ref{tab:unconditional-generation}a).
The largest gain occurs at the smallest budget, where
perplexity falls from 69.169 to 41.260, highlighting
the importance of timestep placement under tight budgets.
Across these budgets, unigram entropy remains within 1.2\%
of the native values, showing that the perplexity reductions
are accompanied by little change in this diversity measure.

AURORA-LM-S exhibits the same benefit at NFE--16
(Table~\ref{tab:unconditional-generation}b).
Perplexity decreases from 56.394 to 41.210, a 26.9\%
reduction, while MAUVE improves from 0.9132 to 0.9326
and entropy remains close to the native value
(5.356 versus 5.386).
The improvement thus extends beyond perplexity to
distributional similarity with human text. Together, these results show that timestep calibration
improves generation quality across architectures and task types.

\subsection{Calibration-query efficiency}
\label{sec:query-efficiency}

We examine how search-space compression affects search
efficiency using the ELF WMT14 De$\rightarrow$En checkpoint
at NFE--16.
We compare three search spaces under the same Bayesian
optimization settings and a budget of 240 queries.
After each query, we select the schedule with the lowest
$F_\alpha$ on the calibration set so far and evaluate it
on the validation set, plotting its measured validation
$F_\alpha$.

Region3 and Mixed6 both divide the sampling interval into
four regions.
Region3 fixes the step allocation to $(4,4,4,4)$ and searches
only the boundaries, whereas Mixed6 also varies the allocations,
giving three and six degrees of freedom, respectively
(Section~\ref{sec:method}).
Full15 searches all 15 interior timesteps without compression.
Each search begins with one uniform schedule and 15 Sobol
points, followed by Bayesian optimization.
The shaded region in Figure~\ref{fig:calibration-query-efficiency}
marks this initialization phase.

\begingroup
\setlength{\columnsep}{6pt}
\setlength{\intextsep}{3pt}

\begin{wrapfigure}{R}{0.50\textwidth}
  \vspace{-0.5\baselineskip}
  \centering

  \includegraphics[
  width=\linewidth,
  keepaspectratio
  ]{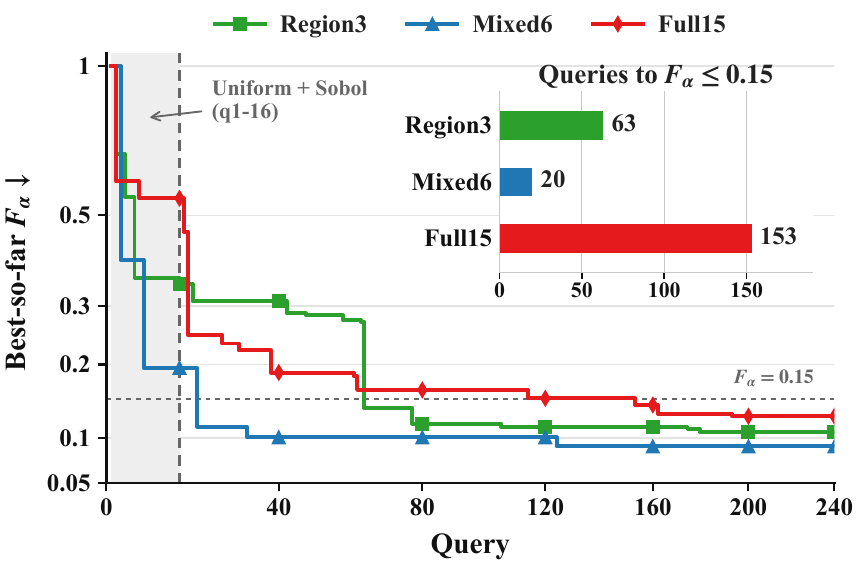}

  \captionsetup{skip=3pt}
  \caption{
  \textbf{Calibration-query efficiency.}
  Validation $F_\alpha$ of the best calibration schedule so far. Both axes use logarithmic compression.
  }
  \label{fig:calibration-query-efficiency}
\end{wrapfigure}

Mixed6 reaches validation $F_\alpha\leq0.15$ after only
20 queries, compared with 63 for Region3 and 153 for Full15,
and achieves the lowest final value within the 240-query budget.
Region3 also outperforms Full15, showing that boundary
compression alone improves search efficiency.
Mixed6's further gains highlight the value of adapting
step allocations alongside boundaries. Together, these results show that greater search-space
flexibility does not necessarily yield better schedules
under a limited query budget.

We also report the actual search time used to obtain
the ELF NFE--16 schedules evaluated in
Table~\ref{tab:conditional-generation}. These schedules were calibrated with 160 queries, taking 0.57 GPU-minutes for WMT14 and
6.79 GPU-minutes for XSum on a single NVIDIA B200. This one-time cost is amortized over subsequent generation
by reusing the frozen schedules.
Timing details are provided in
Appendix~\ref{app:experimental-details}.

\par
\endgroup

\subsection{Ablation Study}
\label{sec:objective-ablation}

\begin{wraptable}{r}{0.47\textwidth}
  \centering
  \captionsetup{
  skip=4pt,
  justification=raggedright,
  singlelinecheck=false
  }
  \caption{\textbf{Objective ablation.} WMT14, NFE--16;
  mean $\pm$ standard deviation across ten calibration panels.}
  \label{tab:objective-ablation}
  \small
  \setlength{\tabcolsep}{3pt}
  \renewcommand{\arraystretch}{1.10}
  \begin{tabular*}{\linewidth}{@{\extracolsep{\fill}}lcc@{}}
    \toprule
    Method & BLEU $\uparrow$ & chrF $\uparrow$ \\
    \midrule
    $A$-only
    & $24.361 \pm 0.153$ & $52.653 \pm 0.108$ \\
    KL-only
    & $24.357 \pm 0.209$ & $52.648 \pm 0.142$ \\
    Pure-$F$
    & $24.515 \pm 0.121$ & $52.771 \pm 0.067$ \\
    Pure-$Q$
    & $23.814 \pm 1.975$ & $52.227 \pm 1.543$ \\
    \addlinespace[2pt]
    \textbf{Full}
    & $\mathbf{24.721} \pm 0.383$
    & $\mathbf{52.918} \pm 0.247$ \\
    \bottomrule
  \end{tabular*}
\end{wraptable}

We compare calibration objectives on WMT14 at NFE--16
using Region3 from Section~\ref{sec:query-efficiency}
and ten independent calibration sets of 16 examples each.
Table~\ref{tab:objective-ablation} reports test-set quality
for the resulting frozen schedules, including variation
across sets.

A-only and KL-only both perform well
(see Appendix~\ref{app:exp-conditional} for detailed definitions),
with A-only achieving slightly higher mean quality
and lower variation.
This is consistent with latent error providing direct
guidance at larger deviations, where it strongly correlates
with KL.
Pure-$F$ improves both mean quality and stability,
supporting the design rationale for the joint objective
in Section~\ref{sec:stable-reduction}.

Pure-$Q$ yields lower mean quality and much larger variation.
Discrete changes in decoded text make task scores an
irregular search signal, while even high calibration scores
can reflect overfitting to the small calibration set
when $F_\alpha$ is ignored.
Full combines $F_\alpha$-guided search with measured quality
and achieves the highest mean BLEU and chrF.
Its variation is higher than Pure-$F$ but substantially
lower than Pure-$Q$, trading some stability for better
generation quality---the ultimate goal of calibration.

%% file: sections/conclusion.tex
\section{Conclusion}

The practical future of continuous DLMs depends on realizing
the efficiency of parallel generation with few sequential
denoising steps.
Our work explores an opportunity shared by these models:
timestep placement remains adjustable after training,
offering a widely applicable, low-cost route to improving
generation without updating model weights.
Our empirical study demonstrates substantial and persistent
effects on terminal representations and decoded token
distributions, and identifies language-specific properties
of the relationship between them.

Building on these findings, StepTune uses a language-aware
objective and search-space compression to calibrate reusable
schedules under generation-quality constraints.
Experiments across architectures and tasks demonstrate
consistent quality improvements that enable lower latency
and higher throughput.
The resulting schedules are reused across inputs,
amortizing the one-time calibration cost over subsequent
generation.
More broadly, this work highlights timestep placement
as a practical source of improvement in existing models,
showing how better use of a limited sampling budget can
help realize the efficiency gains promised by parallel
language generation.

%% file: sections/ai_use_statement.tex
\subsection*{AI Use Statement}
In the preparation of this work, generative AI tools were employed solely as auxiliary aids. Specifically, they assisted with implementing and debugging parts of the codebase by translating algorithmic procedures and design specifications conceived by the authors into working code, and with refining the manuscript's language, including improving fluency, correcting grammatical and typographical errors, and fixing formatting inconsistencies. All scientific contributions of this paper, including the research motivation, the proposed method and its formulation, the theoretical derivations, the experimental protocol, and the interpretation of results, originate entirely from the authors. The AI tools were not involved in conceiving the research direction, designing the methodology or experiments, composing any substantive portion of the manuscript, producing or altering experimental data, establishing theoretical claims, or drawing scientific conclusions. All AI-assisted code and text were carefully inspected, validated, and where necessary revised by the authors prior to inclusion. The authors bear full responsibility for the accuracy and integrity of the published work.

%% file: sections/ethics_statement.tex
\subsection*{Ethics Statement}

StepTune improves the quality--efficiency trade-off of continuous
diffusion language models by optimizing inference timestep placement
while keeping model parameters fixed. Our experiments use existing
model checkpoints and established benchmarks, including XSum,
WMT14, and OpenWebText, with automated evaluation metrics.
Reducing the number of denoising steps can lower inference costs
and make language generation more accessible, but may also lower
the cost of generating harmful or misleading content.
StepTune does not specifically address biases, factual inaccuracies,
or privacy risks associated with the underlying models and their
training data. Its calibration constraints assess generation quality
and do not constitute safety guarantees. Moreover, because timestep
placement changes output distributions, keeping model weights fixed
does not ensure that safety-related behavior remains unchanged.
Deployment should therefore include application-specific safety
evaluation, appropriate safeguards against misuse, and adherence
to the terms of use of the underlying models and datasets.

%% file: sections/reproducibility_statement.tex
\subsection*{Reproducibility Statement}
The experimental settings needed to reproduce all reported results---checkpoint
and sampler configurations, calibration and selection protocols, search budgets,
and evaluation procedures---are specified in
Appendix~\ref{app:experimental-setup}. Every deployed timestep schedule is
released in the accompanying reproducibility package as exact float32 grids,
together with the compact coordinates and the deterministic constructor that
regenerates them; these frozen operating points can be evaluated directly from the released checkpoints
without re-running calibration.
The reproducibility package is publicly available at
\url{https://github.com/collapsar-z/StepTune-reproducibility}.

%% file: sections/appendix.tex
\appendix
\section*{Appendix}
\section{Theoretical Analysis}
\label{app:theory}

This appendix makes precise how timestep placement affects the terminal latent state and its decoded
distribution.  The argument has four parts.  First, we write the self-conditioned sampler as a map on
an augmented state and derive an exact, pathwise decomposition of its terminal error.  Second, we
separate this exact identity from the reference-Jacobian approximation and bound the approximation
remainder.  Third, we connect terminal latent error to decoder KL divergence, both locally and at a
finite displacement.  Finally, we show which quantities must remain stable for one fixed schedule to
retain low teacher-relative error across inputs and random seeds.

Throughout, \(\|\cdot\|_2\) denotes the Euclidean norm and \(\|\cdot\|_{\mathrm{op}}\) its induced
operator norm.  All statements are for a fixed model, sampler, number of function evaluations, and
generation setting.  Random samplers are treated pathwise after their random driver has been fixed.

\subsection{Sampler, schedule, and numerical reference}
\label{app:theory-setup}

Let
\[
\theta=(t_0,\ldots,t_N),
\qquad 0=t_0<t_1<\cdots<t_N=1,
\qquad h_k=t_{k+1}-t_k,
\]
be an \(N\)-step schedule.  A self-conditioned sampler carries both a latent state \(z_k\) and a
memory state \(s_k\), such as the model's previous endpoint prediction.  We therefore use the
augmented state
\[
y_k=(z_k,s_k)\in\mathcal{Y}.
\]
For example, a deterministic Euler step can be written abstractly as
\[
p_k=g_x(z_k,s_k,t_k),
\qquad
z_{k+1}=z_k+h_k v_x(z_k,s_k,t_k),
\qquad
s_{k+1}=p_k,
\]
where \(g_x\) includes the conditioning and guidance operations and \(v_x\) includes the
parameterization-specific conversion from the model prediction to a velocity.  The analysis does not
depend on this particular parameterization.  It only requires the complete deployed update to define
a map
\begin{equation}
  y_{k+1}=\Phi^{\theta}_{x,k}(y_k;\omega),
  \label{eq:realized-step-map}
\end{equation}
where \(\omega\) collects any random variables used by the rollout.  Crucially,
\(\Phi^{\theta}_{x,k}\) includes both the latent update and the memory refresh.  Omitting the latter
would remove a genuine route by which an early error affects later model evaluations.

Let \(y_k^*\) denote a reference state at the same boundary \(t_k\).  In our application these states
come from a refined numerical rollout with matched conditioning, initialization, and sampler
randomness.  The reference need not be the exact solution of a continuous-time process.  Define the
boundary error and one-step defect by
\begin{equation}
  e_k:=y_k-y_k^*,
  \qquad
  d_k^{\theta}:=y_{k+1}^*-
  \Phi^{\theta}_{x,k}(y_k^*;\omega).
  \label{eq:error-and-defect}
\end{equation}
These quantities answer different questions.  The error \(e_k\) records the discrepancy already
present at boundary \(k\).  The defect \(d_k^{\theta}\) starts the coarse update from the reference
state and measures the new discrepancy created on interval \([t_k,t_{k+1}]\).  This distinction is
what permits interval-level attribution.

\subsection{Exact transport of interval defects}
\label{app:exact-transport}

We first state the result without a small-error approximation.

\begin{theorem}[Exact pathwise defect transport]
  \label{thm:exact-transport}\label{thm:transport}
  Fix \(x\), \(\theta\), and the realized driver \(\omega\).  Suppose that
  \(\Phi^{\theta}_{x,k}\) is continuously differentiable on a neighborhood of the line segment joining
  \(y_k^*\) and \(y_k\), for every \(k\).  Define the secant operator
  \begin{equation}
    \bar A_k^{\theta}
    :=\int_0^1
    D_y\Phi^{\theta}_{x,k}(y_k^*+\tau e_k;\omega)\,d\tau
    \label{eq:secant-operator}
  \end{equation}
  and, for \(0\leq j<N\),
  \begin{equation}
    \bar P^{\theta}_{j:N}
    :=\bar A^{\theta}_{N-1}\bar A^{\theta}_{N-2}\cdots\bar A^{\theta}_{j},
    \qquad
    \bar P^{\theta}_{N:N}:=I.
    \label{eq:secant-product}
  \end{equation}
  Then
  \begin{equation}
    e_N
    =\bar P^{\theta}_{0:N}e_0
    -\sum_{k=0}^{N-1}\bar P^{\theta}_{k+1:N}d_k^{\theta}.
    \label{eq:exact-terminal-general}
  \end{equation}
  If the candidate and reference share their augmented initialization, then \(e_0=0\) and
  \begin{equation}
    e_N=-\sum_{k=0}^{N-1}\bar P^{\theta}_{k+1:N}d_k^{\theta}.
    \label{eq:exact-terminal}
  \end{equation}
\end{theorem}

\begin{proof}
  Subtract the reference state at the next boundary and insert the defect from
  Equation~\eqref{eq:error-and-defect}:
  \begin{align}
    e_{k+1}
    &=\Phi^{\theta}_{x,k}(y_k;\omega)-y_{k+1}^* \\
    &=\Phi^{\theta}_{x,k}(y_k;\omega)
    -\Phi^{\theta}_{x,k}(y_k^*;\omega)-d_k^{\theta}.
    \label{eq:one-step-before-ftc}
  \end{align}
  The fundamental theorem of calculus along the segment from \(y_k^*\) to \(y_k\) gives
  \[
  \Phi^{\theta}_{x,k}(y_k;\omega)
  -\Phi^{\theta}_{x,k}(y_k^*;\omega)
  =\bar A_k^{\theta}e_k.
  \]
  Hence
  \begin{equation}
    e_{k+1}=\bar A_k^{\theta}e_k-d_k^{\theta}.
    \label{eq:exact-one-step-recursion}
  \end{equation}
  Repeated substitution proves Equation~\eqref{eq:exact-terminal-general}; setting \(e_0=0\) yields
  Equation~\eqref{eq:exact-terminal}.
\end{proof}

The theorem separates two roles of the schedule.  The defect \(d_k^{\theta}\) measures where the
coarse grid creates error, while \(\bar P^{\theta}_{k+1:N}\) describes how all subsequent updates
carry that error to the endpoint.  An early defect passes through more later operators, but this does
not by itself imply amplification: those operators may amplify, attenuate, rotate, or cancel its
contribution.  The terminal discrepancy therefore depends on both the magnitudes and directions of
the transported defects.

Equation~\eqref{eq:exact-terminal} is exact at the deployed NFE.  It requires neither independent
defects, invertible steps, nor small intermediate errors.  Its pathwise form also allows the transport
operators and defects to be statistically dependent.  Timestep placement influences both objects,
because changing a boundary changes the one-step approximation as well as the states at which later
updates are evaluated.

\subsection{Reference linearization and its remainder}
\label{app:linearized-transport}

The secant operators are exact, but they depend on the complete candidate--reference displacement.
A simpler approximation replaces each secant operator with the Jacobian at the reference state.  The
next result makes the resulting error explicit.

\begin{proposition}[Reference-Jacobian approximation]
  \label{prop:linearized-transport}
  Assume the conditions of Theorem~\ref{thm:exact-transport}.  Suppose additionally that each realized
  step map is twice continuously differentiable on the relevant state segment and that
  \[
  \|D_y\Phi^{\theta}_{x,k}(y;\omega)\|_{\mathrm{op}}\leq L_k,
  \qquad
  \|D_y^2\Phi^{\theta}_{x,k}(y;\omega)\|_{\mathrm{op}}\leq B_k
  \]
  there.  Let
  \[
  A_k^{\theta}:=D_y\Phi^{\theta}_{x,k}(y_k^*;\omega),
  \qquad
  P^{\theta}_{j:N}:=A^{\theta}_{N-1}\cdots A^{\theta}_{j},
  \qquad
  P^{\theta}_{N:N}:=I.
  \]
  For matched initialization,
  \begin{equation}
    e_N
    =-\sum_{k=0}^{N-1}P^{\theta}_{k+1:N}d_k^{\theta}+R_N,
    \label{eq:linearized-terminal}
  \end{equation}
  where
  \begin{equation}
    \|R_N\|_2
    \leq \frac{1}{2}\sum_{k=0}^{N-1}
    \left(\prod_{j=k+1}^{N-1}L_j\right)
    B_k\|e_k\|_2^2.
    \label{eq:linearization-remainder-bound}
  \end{equation}
  An empty product in Equation~\eqref{eq:linearization-remainder-bound} equals one.
\end{proposition}

\begin{proof}
  Taylor's theorem at \(y_k^*\) gives
  \begin{equation}
    e_{k+1}=A_k^{\theta}e_k-d_k^{\theta}+r_k,
    \qquad
    \|r_k\|_2\leq \frac{1}{2}B_k\|e_k\|_2^2.
    \label{eq:linearized-one-step}
  \end{equation}
  Unrolling Equation~\eqref{eq:linearized-one-step} with \(e_0=0\) yields
  \[
  R_N=\sum_{k=0}^{N-1}P^{\theta}_{k+1:N}r_k.
  \]
  The claimed bound follows from the triangle inequality,
  \(\|P^{\theta}_{k+1:N}\|_{\mathrm{op}}
  \leq\prod_{j=k+1}^{N-1}L_j\), and the bound on \(r_k\).
\end{proof}

Equation~\eqref{eq:linearization-remainder-bound} shows exactly when the approximation can fail:
curvature at step \(k\) is multiplied by the squared error already present at that step and then
transported through the remaining dynamics.  Thus the reference-Jacobian approximation is a local
statement.  At a large student--reference gap, the exact secant identity remains valid but the
quadratic remainders need not be small.

\subsection{Random samplers and coarse--fine coupling}
\label{app:stochastic-coupling}

\paragraph{State derivatives of a realized update.}
Fixing the random driver leaves an ordinary map of the augmented state.  For example, an
Euler--Maruyama latent update may take the form
\begin{equation}
  z_{k+1}=z_k+h_k b_k(y_k)
  +\sum_m \sigma_{k,m}(y_k)\,\Delta W_{k,m}.
  \label{eq:euler-maruyama}
\end{equation}
Conditional on the realized increments, its state derivative contains
\begin{equation}
  I+h_k\partial_z b_k
  +\sum_m \Delta W_{k,m}\,\partial_z\sigma_{k,m}.
  \label{eq:euler-maruyama-jacobian}
\end{equation}
The derivatives with respect to the memory state and the derivatives of the memory refresh are
included in the augmented Jacobian.  We differentiate the realized finite-step map with respect to
its state; we do not differentiate a Brownian path with respect to time.  The smoothness assumption
of Theorem~\ref{thm:exact-transport} applies to this complete realized map.

\paragraph{Coupling across grids.}
For a genuine SDE discretization, a common Brownian path can be
constructed on the union grid \(\{\tau_j\}\) of the candidate and reference
boundaries. Draw independent \(\eta_j\sim\mathcal{N}(0,I)\) and set
\begin{equation}
  \Delta W_j=\sqrt{\tau_{j+1}-\tau_j}\,\eta_j,
  \qquad
  \Delta W_k^{\theta}
  =\sum_{j:\,[\tau_j,\tau_{j+1}]\subseteq[t_k,t_{k+1}]}
  \Delta W_j.
  \label{eq:coupling}
\end{equation}
Summing the same fine increments over each coarse interval preserves the correct marginal law on
both grids and produces a common-path coupling.\footnote{M.~B. Giles, Multilevel Monte Carlo Path
Simulation, \emph{Operations Research}, 56(3):607--617, 2008.} Additional grid points
can be introduced using Brownian bridges conditional on the already fixed path values.

For a discrete Gaussian-reinjection sampler, the coupling instead
specifies how a shared collection of standard Gaussian draws is assigned
to the two rollouts while preserving each rollout's marginal law.  Matching draws by their index is
one valid choice.  Reinjection amplitudes and shifted evaluation times remain part of
\(\Phi^{\theta}_{k}\). Such a coupling need not represent one Brownian path: the transport theorem
concerns the deployed finite-step maps and does not require an underlying SDE convergence argument.

\paragraph{Numerical references.}
The reference states in the transport identity may themselves be
numerical. A refinement containing the candidate boundaries, or
an interpolation supplying both latent and memory components,
defines the required states. The resulting reference discrepancy
is absorbed into \(d_k^{\theta}\).
When connecting this decomposition to \(F_\alpha\), the reference
endpoint must be the same \(Z_x^*\) used to evaluate that objective.
The algebra is valid for any such reference under a fixed coupling;
it does not identify a numerical teacher with the exact solution of
a continuous-time process.

\subsection{From terminal latent error to decoder KL}
\label{app:decoder-geometry}

\begin{theorem}[Local decoder pullback]
  \label{thm:decoder}\label{thm:decoder-geometry}
  Let \(Z_x^*\) and \(Z_x(\theta)\) be the reference and candidate terminal latent states for example
  \(x\). At retained token position \(i\), let \(\ell_{i,x}(Z)\in\mathbb{R}^{V}\) be the decoder logits,
  \(q_{i,x}(Z)=\operatorname{softmax}(\ell_{i,x}(Z))\), and
  \(q_{i,x}^*:=q_{i,x}(Z_x^*)\).  With fixed weights \(w_i\geq 0\) satisfying
  \(\sum_i w_i=1\), define the forward teacher-to-candidate KL
  \begin{equation}
    K_x(Z):=\sum_i w_i
    \operatorname{KL}\!\left(q_{i,x}^*\,\middle\|\,q_{i,x}(Z)\right).
    \label{eq:decoder-kl-definition}
  \end{equation}
  If every \(\ell_{i,x}\) is differentiable at \(Z_x^*\), then, as \(e\to 0\),
  \begin{equation}
    K_x(Z_x^*+e)
    =\frac{1}{2}e^\top W_{\mathrm{dec},x}e+o(\|e\|_2^2),
    \label{eq:local-decoder-kl}
  \end{equation}
  where
  \begin{equation}
    W_{\mathrm{dec},x}
    :=\sum_i w_i J_{i,x}^\top F(q_{i,x}^*)J_{i,x},
    \qquad
    J_{i,x}:=D_Z\ell_{i,x}(Z_x^*),
    \label{eq:decoder-pullback}
  \end{equation}
  and
  \[
  F(q):=\operatorname{diag}(q)-qq^\top.
  \]
  In particular, \(W_{\mathrm{dec},x}\) is positive semidefinite.
\end{theorem}

\begin{proof}
  For one position, let \(a^*=\ell_{i,x}(Z_x^*)\),
  \(\delta a=\ell_{i,x}(Z_x^*+e)-a^*\), and
  \(\psi(a)=\log\sum_v\exp(a_v)\).  Direct substitution gives
  \begin{equation}
    \operatorname{KL}(q_{i,x}^*\|q_{i,x}(Z_x^*+e))
    =\psi(a^*+\delta a)-\psi(a^*)-(q_{i,x}^*)^\top\delta a.
    \label{eq:logitkl}
  \end{equation}
  Because \(\nabla\psi(a^*)=q_{i,x}^*\) and
  \(\nabla^2\psi(a^*)=F(q_{i,x}^*)\), Taylor expansion cancels the linear term.  Using
  \(\delta a=J_{i,x}e+o(\|e\|_2)\), then summing over positions, proves
  Equation~\eqref{eq:local-decoder-kl}.  Positive semidefiniteness follows from
  \begin{equation}
    b^\top F(q)b
    =\sum_v q_v b_v^2-\left(\sum_v q_v b_v\right)^2
    =\operatorname{Var}_{v\sim q}[b_v]\geq 0.
    \label{eq:psd}
  \end{equation}
  For sequence outputs, summing with the fixed token weights \(w_i \ge 0\)
  gives the stated \(W_{\mathrm{dec},x}\).
\end{proof}

The pullback in Equation~\eqref{eq:decoder-pullback} identifies latent directions to which the
decoder is locally sensitive.  It is not a finite-distance identity.  The following two results
describe the measured KL away from the reference without replacing it by the local Fisher
quadratic.

\begin{proposition}[Finite logit-displacement bound]
  \label{prop:finite-logit-bound}
  For natural logarithms and arbitrary candidate logits,
  \begin{equation}
    K_x(Z)
    \leq\frac{1}{4}\sum_i w_i
    \|\ell_{i,x}(Z)-\ell_{i,x}(Z_x^*)\|_2^2.
    \label{eq:finite-logit-bound}
  \end{equation}
\end{proposition}

\begin{proof}
  For one position, Taylor's integral formula applied to \(\psi\) gives
  \[
  \operatorname{KL}(q^*\|q)
  =\int_0^1(1-\tau)\,
  \delta a^\top F(\operatorname{softmax}(a^*+\tau\delta a))
  \delta a\,d\tau.
  \]
  The absolute row sums of \(F(q)\) are \(2q_v(1-q_v)\leq 1/2\), so
  \(\|F(q)\|_{\mathrm{op}}\leq 1/2\).  Integrating \(1-\tau\) over \([0,1]\) gives the factor
  \(1/4\).  Summing with the fixed weights proves Equation~\eqref{eq:finite-logit-bound}.
\end{proof}

\begin{proposition}[Exact decomposition along a finite latent path]
  \label{prop:path-decoder}
  Let \(e=Z_x(\theta)-Z_x^*\), \(\gamma(\tau)=Z_x^*+\tau e\), and assume the decoder logits are
  twice continuously differentiable along \(\gamma\).  Write
  \[
  \ell_{i,\tau}:=\ell_{i,x}(\gamma(\tau)),
  \qquad
  q_{i,\tau}:=\operatorname{softmax}(\ell_{i,\tau}),
  \qquad
  q_i^*:=q_{i,0},
  \]
  and let dots denote derivatives with respect to \(\tau\).  Define
  \begin{align}
    g_\gamma(\tau)
    &:=\sum_i w_i\dot\ell_{i,\tau}^\top
    F(q_{i,\tau})\dot\ell_{i,\tau}, \\
    r_\gamma(\tau)
    &:=\sum_i w_i(q_{i,\tau}-q_i^*)^\top\ddot\ell_{i,\tau}.
    \label{eq:path-terms}
  \end{align}
  Then
  \begin{equation}
    K_x(Z_x(\theta))=G_\gamma+R_\gamma,
    \qquad
    G_\gamma:=\int_0^1(1-\tau)g_\gamma(\tau)\,d\tau,
    \qquad
    R_\gamma:=\int_0^1(1-\tau)r_\gamma(\tau)\,d\tau.
    \label{eq:path-decomposition}
  \end{equation}
  Here \(G_\gamma\geq 0\), whereas \(R_\gamma\) can have either sign.
\end{proposition}

\begin{proof}
  Let \(K(\tau):=K_x(\gamma(\tau))\).  Direct differentiation gives
  \[
  K'(\tau)=\sum_i w_i(q_{i,\tau}-q_i^*)^\top\dot\ell_{i,\tau}
  \]
  and therefore \(K''(\tau)=g_\gamma(\tau)+r_\gamma(\tau)\).  Since
  \(K(0)=K'(0)=0\), Taylor's integral formula yields
  Equation~\eqref{eq:path-decomposition}.  The nonnegativity of \(G_\gamma\) follows from the
  positive semidefiniteness of \(F(q)\).
\end{proof}

Theorem~\ref{thm:decoder} describes the infinitesimal geometry at the reference.
Proposition~\ref{prop:path-decoder} is an exact identity for a finite latent displacement and exposes
the additional curvature of the decoder logits along that path.  Proposition~\ref{prop:finite-logit-bound}
is a global upper bound in logit space.  These statements are complementary and should not be used
interchangeably.

\subsection{Terminal normalization and the calibration objective}
\label{app:calibration-objective}

Let \(M_x\) select the terminal latent region on which fidelity is evaluated.  For unconditional generation, it selects all generated latent positions. For conditional generation, it selects the generated
suffix and excludes conditioning positions.  With a fixed floor \(\tau_A>0\), define
\begin{equation}
  A_x(\theta)
  :=\frac{\|M_x(Z_x(\theta)-Z_x^*)\|_2^2}{D_{A,x}},
  \qquad
  D_{A,x}:=\max\{\|M_x Z_x^*\|_2^2,\tau_A\}.
  \label{eq:normalized-latent-error}
\end{equation}
For the decoder term, the retained positions are fixed by the teacher: we include positions through
the teacher's first EOS token, inclusive, or the full generated region if no EOS appears.  Retained
positions have uniform weights summing to one.  This teacher-defined mask prevents an early student
termination from shortening its own evaluation range.

For a calibration set \(\mathcal{C}\), let
\[
\bar A(\theta):=\frac{1}{|\mathcal{C}|}\sum_{x\in\mathcal{C}}A_x(\theta),
\qquad
\bar K(\theta):=\frac{1}{|\mathcal{C}|}\sum_{x\in\mathcal{C}}
K_x(Z_x(\theta)).
\]
Given fixed positive scales \(A_{\mathrm{scale}}\) and \(K_{\mathrm{scale}}\), StepTune uses
\begin{equation}
  F_\alpha(\theta)
  :=\alpha\frac{\bar A(\theta)}{A_{\mathrm{scale}}}
  +(1-\alpha)\frac{\bar K(\theta)}{K_{\mathrm{scale}}},
  \qquad 0\leq\alpha\leq 1.
  \label{eq:calibration-objective}
\end{equation}
The masks, denominator floors, and scales are held fixed across candidate schedules.

For one example, define
\[
W_{A,x}:=\frac{M_x^\top M_x}{D_{A,x}},
\qquad
W_{\alpha,x}
:=\frac{2\alpha}{A_{\mathrm{scale}}}W_{A,x}
+\frac{1-\alpha}{K_{\mathrm{scale}}}W_{\mathrm{dec},x}.
\]
Combining the per-example objective
\(F_{\alpha,x}(\theta):=\alpha\,A_x(\theta)/A_{\mathrm{scale}}
+(1-\alpha)\,K_x(Z_x(\theta))/K_{\mathrm{scale}}\)
with Theorem~\ref{thm:decoder} yields the local form
\begin{equation}
  F_{\alpha,x}(\theta)
  =\frac{1}{2}(e_x^Z)^\top W_{\alpha,x}e_x^Z
  +o(\|e_x^Z\|_2^2).
  \label{eq:local-joint-objective}
\end{equation}
The latent term penalizes every selected latent direction, including directions to which the decoder
is locally insensitive; the KL term weights directions by their effect on token probabilities.

Let \(Q_z\) project the augmented terminal error onto its latent component.  Under matched
initialization, Theorem~\ref{thm:exact-transport} gives
\[
e_x^Z=-\sum_{k=0}^{N-1}u_{k,x}^{\theta},
\qquad
u_{k,x}^{\theta}:=Q_z\bar P^{\theta}_{k+1:N}d_k^{\theta}.
\]
Substitution into Equation~\eqref{eq:local-joint-objective} makes the interaction among intervals
explicit:
\begin{equation}
  F_{\alpha,x}(\theta)
  =\frac{1}{2}\sum_{k=0}^{N-1}\sum_{\ell=0}^{N-1}
  (u_{k,x}^{\theta})^\top W_{\alpha,x}u_{\ell,x}^{\theta}
  +o(\|e_x^Z\|_2^2).
  \label{eq:transported-defect-objective}
\end{equation}
The diagonal terms measure the contribution of individual transported defects; the off-diagonal
terms record reinforcement or cancellation between intervals.  StepTune evaluates the actual latent
error and decoder KL in Equation~\eqref{eq:calibration-objective}.  Equations
\eqref{eq:local-joint-objective}--\eqref{eq:transported-defect-objective} explain the geometry of
those two measurements but do not replace them during search.

\subsection{When a fixed schedule can transfer}
\label{app:schedule-reuse}

A reusable schedule need not produce similar terminal states for different inputs.  It must keep each
candidate trajectory close to its own paired reference.  We formalize this distinction for a fixed
schedule \(\theta\).

Let \(\xi=(x,z_0,\omega)\) denote the conditioning input, initial latent, and subsequent random
driver.  For two realizations \(\xi\) and \(\xi'\) with compatible state coordinates, define
\[
\Delta e_k:=e_k(\xi)-e_k(\xi'),
\quad
\Delta d_k:=d_k(\xi)-d_k(\xi'),
\quad
\Delta\bar A_k:=\bar A_k(\xi)-\bar A_k(\xi').
\]

\begin{proposition}[Cross-realization stability bound]
  \label{prop:reuse-bound}
  Suppose the conditions of Theorem~\ref{thm:exact-transport} hold for both realizations and each
  candidate is paired with a reference that shares its own augmented initialization.  Then
  \begin{equation}
    \|\Delta e_N\|_2
    \leq\sum_{k=0}^{N-1}
    \|\bar P_{k+1:N}(\xi)\|_{\mathrm{op}}
    \left(
    \|\Delta d_k\|_2
    +\|\Delta\bar A_k\|_{\mathrm{op}}\|e_k(\xi')\|_2
    \right).
    \label{eq:pairbound}
  \end{equation}
\end{proposition}

\begin{proof}
  Subtract the two exact one-step recursions and add and subtract
  \(\bar A_k(\xi)e_k(\xi')\):
  \begin{equation}
    \Delta e_{k+1}
    =\bar A_k(\xi)\Delta e_k
    +\Delta\bar A_k e_k(\xi')-\Delta d_k.
    \label{eq:pairsplit}
  \end{equation}
  Although \(z_0\) may differ across realizations, each candidate and its own reference share their
  initial state.  Hence \(e_0(\xi)=e_0(\xi')=0\) and \(\Delta e_0=0\).  Unrolling
  Equation~\eqref{eq:pairsplit}, followed by the triangle inequality and submultiplicativity,
  proves Equation~\eqref{eq:pairbound}.
\end{proof}

The bound identifies a sufficient mechanism for reuse.  If the transported change in interval defects
is controlled, the transport operators do not vary too sharply, and the paired intermediate errors
remain small, then teacher-relative terminal errors can remain similar across otherwise different
trajectories.  Operator variation is multiplied by the paired error \(e_k(\xi')\), not by the absolute
reference state.  This permits a fixed grid to address recurring numerical-error structure while the
input and random seed continue to determine the generated content.

This is a conditional explanation, not a distribution-free guarantee.  Transfer of the calibration
objective additionally requires comparable normalization and decoder sensitivity across examples;
transfer to language quality requires the quality constraints and held-out evaluations described in
the main text.

\section{Experimental Details}
\label{app:experimental-setup}
\label{app:experimental-details}

This appendix gives the calibration and evaluation protocols for the reported generation results.
Calibration data are disjoint from validation and test data. Each selected schedule is frozen
before held-out evaluation and reused across inputs.

\subsection{Models, precision, and controlled variables}
\label{app:exp-common}
\label{app:models-sampling}

\paragraph{ELF checkpoints.}
The cross-scale and unconditional experiments use the PyTorch ELF checkpoints trained on
OpenWebText.  The ELF-B, ELF-M, and ELF-L checkpoints are taken at training steps \(95{,}085\),
\(76{,}070\), and \(57{,}051\), respectively.  Conditional generation uses the task-specific ELF-B
checkpoints for XSum at step \(39{,}800\) and WMT14 De\(\rightarrow\)En at step \(880{,}600\),
together with a frozen T5-small encoder.  ELF generation uses the T5-small tokenizer.  Standard
generation rollouts use bfloat16 model computation and float32 timestep grids.

\paragraph{What changes across candidate schedules.}
Within each deployment setting, all candidates use the same frozen checkpoint, sampler update,
guidance parameters, endpoint times, and number of model evaluations.  The candidate variable is
the placement of the interior timesteps.  Candidate schedules evaluated on the same calibration
example share its conditioning input and initial latent.  Random samplers additionally use the fixed
coupling specified for that experiment.  Thus paired differences within a setting isolate timestep
placement; the higher-NFE teacher is a separate reference configuration and may use a different NFE
or stochastic coefficient, as stated below.

\paragraph{Calibration, selection, and evaluation.}
One calibration query evaluates one complete grid on every example in the calibration panel.  Teacher
terminal latents and token distributions are computed once and cached.  Schedule search and all
quality gates use only the calibration panel.  The selected schedule is then frozen and evaluated on
disjoint validation seeds or held-out task examples. If no candidate satisfies the stated
feasibility requirements, the baseline schedule is retained.

We use $F=F_{1/2}$ from Equation~\eqref{eq:joint-discrepancy}.
Its components are the calibration means of relative squared terminal-latent error
and forward decoder KL, $\mathrm{KL}(q_{\mathrm{teacher}}\|q_{\mathrm{candidate}})$.
The latent region, teacher-EOS mask, and per-example normalizations follow
Appendix~\ref{app:theory}. The two normalization constants are fixed before
comparing candidates; their sources are specified below.

\subsection{Schedule parameterization and search budget}
\label{app:exp-search}

\paragraph{Four-region schedules.}
Conditional generation and the cross-scale control experiment use four contiguous trajectory
regions.  A softmax maps four boundary logits \(\beta=(\beta_1,\ldots,\beta_4)\), with the gauge
\(\beta_1=0\), to positive region widths and cumulative boundaries \(B_r\).  When allocations are
optimized, three additional free logits determine a four-way allocation softmax: we first assign one
interval to every region and distribute the remaining \(N-4\) intervals by largest-remainder
rounding, giving positive integers \(n_r\) with \(\sum_r n_r=N\).  Region \(r\) then contains
\(n_r\) equal-width intervals on \([B_{r-1},B_r]\), and the union of their left endpoints, followed
by the terminal endpoint \(1\), forms a grid with exactly \(N\) solver intervals.

\paragraph{Bayesian search for ELF.}
Unless a setting below fixes the allocation, the search has six free coordinates: three boundary
coordinates in \([-3,3]\) and three allocation coordinates in \([-3.5,3.5]\).  Every search begins
with Uniform and 15 scrambled Sobol proposals.  Later queries are issued in batches of four using a
Gaussian-process lower-confidence-bound acquisition function with a Mat\'ern-\(5/2\) kernel,
length scale \(0.25\), and exploration coefficient \(1.5\).  Teacher fidelity is the acquisition
objective.  In task experiments, measured quality constraints determine feasibility and may also
penalize predicted violations.  The reported schedule is selected only from fully evaluated
candidates.

\paragraph{Finite-bank search.}
ELF-B unconditional generation does not use the continuous parameterization.  At each NFE it evaluates a
fixed bank of 80 grids drawn from ELF's native LogitNormal grid family.

Table~\ref{app:tab:calibration-settings} summarizes the teacher and query budgets.  A budget is reported
per deployment NFE; each grid evaluation covers the entire calibration panel.

\begin{table}[!ht]
  \centering
  \caption{\textbf{ELF calibration configurations.}
  Budgets are per NFE. Each grid evaluation covers the full
  calibration panel; ELF-B unconditional calibration evaluates a fixed bank.}
  \label{app:tab:calibration-settings}
  \small
  \setlength{\tabcolsep}{4pt}
  \renewcommand{\arraystretch}{1.08}
  \begin{tabular*}{\linewidth}{@{\extracolsep{\fill}}lcccc@{}}
    \toprule
    Experiment & NFE & Panel size & Grid evaluations & Teacher NFE \\
    \midrule
    Schedule control & 16/32 & 16 & 160 & 64 \\
    Unconditional & 8/16/32 & 64 & 80 (bank) & 64 \\
    XSum & 16/32 & 64 & 160 & 64 \\
    XSum & 64 & 64 & 160 & 128 \\
    WMT14 & 16/32 & 16 & 160 & 64 \\
    WMT14 & 64 & 16 & 160 & 128 \\
    \bottomrule
  \end{tabular*}
\end{table}

\subsection{Cross-scale schedule-control experiment}
\label{app:exp-cross-scale}
\label{app:control-settings}

This experiment asks whether a schedule calibrated on one ELF model scale retains low
teacher-relative error on other scales.  We evaluate ELF-B, ELF-M, and ELF-L at deployment NFE
16 and 32.  All runs use sequence length \(1{,}024\), Euler ODE sampling, classifier-free-guidance
scale \(1\), self-conditioning-guidance scale \(3\), and noise scale \(2\).  The teacher uses the
same checkpoint as its paired candidate and a uniform NFE-64 grid.

For each deployment NFE, the calibration panel contains 16 ELF-M seeds and the validation panel
contains 160 disjoint seeds.  We run one 160-query search on ELF-M and retain the observed grid with
the lowest calibration \(F\); no task-quality criterion is used.  This produces exactly two schedules,
one for NFE 16 and one for NFE 32.  Each schedule is then applied unchanged to ELF-B, ELF-M, and
ELF-L on the validation panel.  No schedule is recalibrated for ELF-B or ELF-L.

For the latent term, terminal MSE is normalized by the paired teacher's terminal energy.  The decoder
term uses the forward KL through the teacher's first EOS token, inclusive.  For reporting on a given
model, NFE, and split, each component is divided by its mean over 32 native LogitNormal replays;
the native mean of the resulting joint objective is therefore one.  Search itself uses only the ELF-M
calibration normalizers.  Confidence intervals use 16,384 crossed-bootstrap draws: validation seeds
and native replays are resampled independently, and the reporting normalizers are recomputed within
each draw.

\subsection{Unconditional generation on ELF-B}
\label{app:exp-unconditional}
\label{app:unconditional-settings}

\paragraph{Sampling and calibration.}
We use ELF-B with sequence length \(1{,}024\) and the official SDE-inspired sampler.
Classifier-free-guidance and self-conditioning-guidance scales are \(1\) and \(3\), respectively.
The stochastic coefficient is \(\gamma=2.0\) at deployment NFE 8 and 16 and \(\gamma=1.5\) at
NFE 32; the final update is deterministic.  The paired teacher uses a uniform NFE-64 grid and
\(\gamma=1\).

For each deployment NFE, the calibration panel contains 64 seeds.  We evaluate 80 native
LogitNormal grids with parameters \((-1.5,0.8)\).  The bank means of latent error and decoder KL
define \(A_{\mathrm{scale}}\) and \(K_{\mathrm{scale}}\), respectively.  A grid is feasible only if
\(F<1\) and it satisfies all applicable requirements in Table~\ref{app:tab:quality-gates}.  Among the
feasible grids, let \(\mathrm{PPL}_{\min}\) be the lowest calibration GPT-2-Large perplexity.  We
retain grids with perplexity at most \(1.05\,\mathrm{PPL}_{\min}\), then select the retained grid with
the lowest \(F\).  Exact ties are resolved first by lower perplexity and then by candidate identifier.
One schedule is selected at each NFE.

\paragraph{Held-out evaluation.}
For every deployment NFE, the frozen \method{} schedule and the matched native baseline are each
evaluated in six independent runs of 1,000 generations.  At batch size one, the native baseline draws
a fresh LogitNormal grid for every example, while \method{} reuses its one frozen grid.  Both methods
use the same checkpoint, sampler configuration, and NFE.

We report GPT-2-Large perplexity and unigram entropy under the GPT-2 tokenizer.  We additionally
record repeated-trigram rate, distinct bigrams, generated length, EOS rate, and empty-output rate.
Paired uncertainty for perplexity is computed in token-negative-log-likelihood space and transformed
back to the reported scale.  The hierarchical bootstrap uses 16,384 replicates, resampling runs and
then examples within each sampled run.  The matched PyTorch baseline is the basis of the controlled comparison.

\begin{table}[!ht]
  \centering
  \caption{\textbf{ELF calibration quality requirements.}
  Subscript \(0\) denotes the bank reference for unconditional
  generation and the uniform student schedule for conditional generation.}
  \label{app:tab:quality-gates}
  \small
  \setlength{\tabcolsep}{4pt}
  \renewcommand{\arraystretch}{1.15}
  \begin{tabular*}{\linewidth}{@{\extracolsep{\fill}}lcc@{}}
    \toprule
    Quantity & Unconditional & Conditional \\
    \midrule
    Unigram entropy \(H\) & \(\geq 0.98H_0\) & -- \\
    Relative length \(L/L_0\) & \([0.90,1.10]\) & \([0.80,1.20]\) \\
    Repeated-trigram rate \(R_3\)
    & \(\leq\max(1.10R_{3,0},R_{3,0}+0.01)\)
    & \(\leq\max(1.50R_{3,0},R_{3,0}+0.01)\) \\
    EOS rate & \(\geq\mathrm{EOS}_0-0.05\)
    & \(\geq\mathrm{EOS}_0-0.05\) \\
    Empty-output rate & \(\leq\mathrm{Empty}_0+0.01\)
    & \(\leq\mathrm{Empty}_0+0.01\) \\
    \bottomrule
  \end{tabular*}
\end{table}

\subsection{Conditional generation on ELF-B}
\label{app:exp-conditional}
\label{app:conditional-settings}

\paragraph{Shared sampling configuration.}
XSum and WMT14 De\(\rightarrow\)En keep the conditioning positions fixed and generate a
64-token suffix.  Candidate rollouts use Euler ODE sampling, classifier-free-guidance scale \(2\),
self-conditioning-guidance scale \(1\), and decoder scale \(1\).  XSum uses 64 training pairs for
calibration; WMT14 uses 16.  The paired teacher is Uniform at NFE 64 for deployment NFE 16 and 32,
and Uniform at NFE 128 for deployment NFE 64.  For each task and NFE, the corresponding Uniform
student's mean latent error and decoder KL define the two scales in
Equation~\eqref{eq:joint-discrepancy}.

\paragraph{Search space and budget.}
All six conditional settings use Mixed6, optimizing both the three interior
region boundaries and the step allocations. Each search evaluates 160 grids
using the initialization and Gaussian-process configuration in
Appendix~\ref{app:exp-search}.

\paragraph{Feasibility.}
At calibration time, every XSum ROUGE-1/2/L score, or both WMT14 sentence BLEU and sentence
chrF, must be no more than \(0.5\) points below Uniform.  A candidate must also satisfy the
task-applicable requirements in Table~\ref{app:tab:quality-gates}.  Rather than accepting a constraint
from its point estimate alone, we require at least \(0.90\) paired-bootstrap support for the joint
feasibility event under 8,192 fixed bootstrap draws.

\paragraph{Final selection.}
Selection is performed only among feasible, fully evaluated candidates.  For XSum, the scalar
calibration-quality score is the equal-weight mean of Uniform-normalized ROUGE-1/2/L.  For WMT14,
it is the equal-weight mean of Uniform-normalized sentence BLEU, sentence chrF, corpus BLEU, and
corpus chrF.  Let the common best candidate maximize this score over the feasible observations.
The XSum near-best set retains a candidate only if each of its three ROUGE metrics has bootstrap
probability at least \(0.90\) of being no more than \(0.5\) points below the corresponding metric of
the common best candidate.  WMT14 applies the same \(0.90\) threshold to the joint BLEU/chrF event.
From this near-best set, we choose the candidate with the lowest \(F\), breaking exact ties by higher
calibration-quality score and then earlier query order.

\paragraph{Test evaluation.}
The selected schedule is evaluated on the complete XSum test set of 11,334 examples and the complete
WMT14 De\(\rightarrow\)En test set of 3,003 examples.  XSum reports ROUGE-1/2/L F1 with
stemming.  WMT14 reports SacreBLEU corpus BLEU and chrF.  The native baseline draws one
LogitNormal grid with parameters \((-1.5,0.8)\) for each ordered batch of 64; the \method{} grid is
fixed across all batches.  Paired comparisons use the same test examples and initial latents.  We use
16,384 item-bootstrap replicates, based on per-example ROUGE for XSum and sentence-level
translation scores for WMT14.  Repeated-trigram rate, length, EOS rate, and empty-output rate are
recorded as generation diagnostics.

\paragraph{Search time for the deployed schedules.}
The ELF NFE--16 schedules evaluated in Table~\ref{tab:conditional-generation} were obtained
with 160 search queries. On a single NVIDIA B200, search takes 0.57 GPU-minutes for WMT14
with 16 calibration examples and 6.79 GPU-minutes for XSum with 64 examples.
We measure wall-clock time with the model resident in accelerator memory, including schedule
rollouts, metric computation, and Bayesian optimization bookkeeping. The measurements exclude
checkpoint loading, reference-cache construction, and held-out evaluation. The selected schedules
are frozen and reused, so this one-time search cost is amortized over subsequent generation.

\paragraph{Inference latency and throughput.}
\label{app:exp-efficiency}
\label{app:efficiency-settings}
Latency and throughput are measured on one NVIDIA B200 with the model resident in accelerator
memory. Latency uses batch size one; throughput uses batch size 64. Five warm-up blocks precede
30 measured blocks, and method order is randomized across measured blocks. We synchronize the
accelerator around every timed block.

The timing boundary includes input transfer, condition encoding, latent construction, sampling,
model decoding, and CPU text decoding. It excludes model loading, dataset I/O, input tokenization,
grid lookup and transfer, metric computation, and offline calibration. Every compared schedule uses
the same timing boundary. Offline search cost is reported separately above and is not included in
per-example latency or throughput.

\paragraph{LD3 adaptation.}
LD3 is evaluated at deployment NFE 16, 32, and 64 with the same teacher NFEs used above.  For
each task, its 64-example calibration set is split equally between parameter fitting and checkpoint
selection.  The adaptation optimizes interval boundaries and model-evaluation times while preserving
conditioning positions and restricting terminal \(\ell_2\) error to the generated suffix.  The first
two epochs optimize only the boundaries; the next five optimize both grids.  Boundary parameters use
RMSprop with learning rate \(0.005\) and momentum \(0.9\).  Evaluation-time offsets use SGD with
learning rate \(0.1\).  The gradient norm is clipped at \(1\).  Optimization and checkpoint selection
use only terminal latent \(\ell_2\) error.  Image-specific LPIPS and relaxed-initial-noise operations
are omitted.

\paragraph{Objective-ablation variants.}
In Section~\ref{sec:objective-ablation}, A-only uses latent error $A$ to guide search
and selects the candidate with the lowest $A$; KL-only applies the same rule
to decoder KL. Pure-$F$ searches and selects using the joint discrepancy $F$,
whereas Pure-$Q$ searches for and selects the highest calibration quality $Q$.
Full uses the complete \method{} procedure: $F$ guides search under quality
constraints, and the lowest-$F$ schedule is selected from the feasible
near-best-quality set.

\paragraph{WMT14 BLEU conventions and aligned rescoring.}
For comparability with the BLEU scores reported in ELF, the main WMT14 experiments
follow ELF's evaluation implementation and compute SacreBLEU corpus BLEU with
\texttt{lowercase=True} and \texttt{use\_effective\_order=True}. The component ablation
uses the default SacreBLEU corpus configuration, with \texttt{lowercase=False} and
\texttt{use\_effective\_order=False}. Both configurations use the default \texttt{13a}
tokenization and exponential smoothing. All ablation variants are evaluated under the
same configuration on the complete 3,003-example WMT14 De\(\rightarrow\)En test set.
Comparisons within the ablation therefore use a consistent scoring convention;
absolute BLEU values should not be compared directly between the main results and
the ablation.

We additionally rescore the saved test outputs from all ten ablation panels using
SacreBLEU 2.6.0 and the main-results BLEU configuration, without changing schedules
or regenerating text. Table~\ref{app:tab:wmt-ablation-aligned} reports the aligned
results. The ordering by mean BLEU remains unchanged: Full, Pure-$F$, A-only,
KL-only, and Pure-$Q$, from highest to lowest. The component-ablation conclusions
thus also hold under the main-results scoring convention. The chrF configuration
is identical throughout, and the original chrF values are retained.

\begin{table}[!ht]
  \centering
  \caption{\textbf{Component ablation under the main WMT14 scoring convention.}
  Results are computed from the saved outputs on the complete 3,003-example test
  set. Entries are means \(\pm\) sample standard deviations across ten calibration
  panels; BLEU is rescored using the main-results configuration, while chrF is unchanged.}
  \label{app:tab:wmt-ablation-aligned}
  \small
  \setlength{\tabcolsep}{4pt}
  \renewcommand{\arraystretch}{1.08}
  \begin{tabular*}{\linewidth}{@{\extracolsep{\fill}}lcc@{}}
    \toprule
    Variant & BLEU & chrF \\
    \midrule
    A-only   & \(25.358 \pm 0.154\) & \(52.653 \pm 0.108\) \\
    KL-only  & \(25.354 \pm 0.222\) & \(52.648 \pm 0.142\) \\
    Pure-$F$ & \(25.519 \pm 0.122\) & \(52.771 \pm 0.067\) \\
    Pure-$Q$ & \(24.796 \pm 2.042\) & \(52.227 \pm 1.543\) \\
    Full     & \(\mathbf{25.731} \pm 0.400\) & \(\mathbf{52.918} \pm 0.247\) \\
    \bottomrule
  \end{tabular*}
\end{table}

\subsection{Generation experiments on AURORA-LM-S}
\label{app:exp-aurora}

These experiments assess \method{} on a second continuous DLM architecture using the released
AURORA-LM-S checkpoints. Both conditional and unconditional generation use the native tan-$d$
schedule as the baseline. We retain the model-specific sampling and decoding settings below.

\paragraph{XSum sampling and calibration.}
We use the released XSum checkpoint at training step 300,000 and the official conditional
preprocessing. Candidate schedules use Euler ODE sampling at NFE 16. The paired reference uses
the same frozen model and the native tan-$d$ schedule at NFE 32. We retain classifier-free-guidance
scale 2.0, the guidance noise window \([0.9,1.0]\), repainting, and a generated suffix of 64 tokens.
Calibration uses 16 examples and 160 queries in the six-dimensional four-region boundary/allocation
space. The two components of \(F\) are normalized by their calibration means under native
tan-$d$ sampling at NFE 16.

The initial design contains the native schedule and 15 scrambled Sobol proposals. Subsequent
proposals use a Gaussian-process lower confidence bound with an RBF kernel, length scale 0.22,
jitter \(10^{-6}\), and exploration coefficient 1.8. Proposals are selected from a global Sobol pool
augmented with local proposals around strong observed schedules. The search is guided by \(F\),
with measured quality violations penalized during calibration.

\paragraph{XSum quality gates and final selection.}
For this supporting evaluation, we use fixed baseline-relative quality gates: each of ROUGE-1,
ROUGE-2, and ROUGE-L must be at least 85\% of its value under native tan-$d$ sampling at NFE 16,
and the empty-output rate must not exceed 5\%. Among the fully evaluated schedules satisfying
these gates, we select the one with the lowest \(F\). No additional near-best-quality tolerance
is applied in this experiment. We freeze the selected schedule and reuse it in five independent
evaluation runs of 1,000 XSum examples, reporting mean ROUGE-1/2/L. The comparison in
Table~\ref{tab:unconditional-generation}b uses the native tan-$d$ schedule at NFE 32.

\paragraph{OpenWebText sampling and calibration.}
We use the released OpenWebText checkpoint at training step 250,000. Candidate and reference
rollouts use the official SDE-DPM++ sampler with self-conditioning-guidance scale 4.0 and
\texttt{eos\_wrap}. Deployment uses NFE 16; the reference uses the native tan-$d$ schedule at
NFE 64. Search uses 16 calibration generations, 40 queries, and the same six-dimensional
four-region search space. Initialization contains the native schedule and 15 scrambled Sobol points,
followed by RBF Gaussian-process lower-confidence-bound proposals. Search generations use
a maximum length of 256 tokens, while final evaluation uses the official 1,024-token protocol.

Relative to the native NFE--16 calibration baseline, the perplexity ratio must not exceed 1.35,
the unigram-entropy decrease must not exceed 0.15, the repeated-trigram-rate increase must not
exceed 0.05, and the empty-output rate must not exceed 0.01. These quality requirements constrain
the \(F\)-guided search. The selected schedule is frozen before evaluation.
\paragraph{OpenWebText evaluation.}
We evaluate the frozen schedule and native NFE--16 baseline with five independent seeds and
1,000 generations per seed. Perplexity is computed with GPT-2-Large, and unigram entropy uses
the GPT-2 tokenizer. MAUVE compares the generated text with the first 1,000 held-out OpenWebText
documents following the released AURORA evaluation recipe. We report the mean across the five
runs and additionally record repeated-trigram and empty-output rates as diagnostics.

\section{The Structure of the Calibrated Schedules}
\label{app:schedule-structure}

We summarize the eleven deployed schedules for the main generation results
using the fraction of evaluation times before $t=0.10$, $0.25$, and $0.50$.
Here, $t$ increases from the initial-noise endpoint $0$ to the data endpoint $1$.
For an NFE-$N$ grid, we count its $N$ left endpoints
$t_0,\ldots,t_{N-1}$, including $t_0=0$ and excluding the terminal endpoint $t_N=1$.
The final interval still ends at $1$, even when all evaluation times lie earlier.

\begin{table}[t]
  \centering
  \caption{\textbf{Temporal allocation of the deployed schedules.}
  Entries are percentages of the $N$ model-evaluation times strictly below
  each threshold, computed from the saved float32 grids.}
  \label{app:tab:schedule-concentration}
  \small
  \setlength{\tabcolsep}{5pt}
  \renewcommand{\arraystretch}{1.08}
  \begin{tabular*}{\linewidth}{@{\extracolsep{\fill}}llrrrr@{}}
    \toprule
    Model & Task & NFE & $t<0.10$ & $t<0.25$ & $t<0.50$ \\
    \midrule
    ELF-B & XSum & 16 & 93.8 & 100.0 & 100.0 \\
    ELF-B & XSum & 32 & 90.6 & 90.6 & 90.6 \\
    ELF-B & XSum & 64 & 92.2 & 92.2 & 92.2 \\
    ELF-B & WMT14 & 16 & 93.8 & 93.8 & 93.8 \\
    ELF-B & WMT14 & 32 & 78.1 & 81.3 & 87.5 \\
    ELF-B & WMT14 & 64 & 90.6 & 93.8 & 96.9 \\
    \midrule
    ELF-B & OpenWebText & 8 & 25.0 & 50.0 & 87.5 \\
    ELF-B & OpenWebText & 16 & 18.8 & 75.0 & 100.0 \\
    ELF-B & OpenWebText & 32 & 15.6 & 78.1 & 96.9 \\
    \midrule
    AURORA-LM-S & XSum & 16 & 37.5 & 87.5 & 87.5 \\
    AURORA-LM-S & OpenWebText & 16 & 6.3 & 25.0 & 81.3 \\
    \bottomrule
  \end{tabular*}
\end{table}

The ELF conditional schedules consistently concentrate evaluations near the
beginning of the trajectory: 78.1--93.8\% occur before $t=0.10$ across the
six settings. Their precise placements differ, but the preference for early
evaluations recurs across tasks and budgets. The unconditional schedules
also allocate 50.0--78.1\% of evaluations before $t=0.25$, although these
were selected from an already early-skewed LogitNormal candidate family.

AURORA-LM-S shows a more task-dependent pattern. Its XSum schedule places
87.5\% of evaluations before $t=0.25$. On OpenWebText, only 25.0\% lie before
that threshold, while 56.3\% lie in $[0.25,0.50)$. Thus, early concentration
is common in the tested settings, particularly on ELF, but is not a universal
shape shared by all selected schedules. These observations describe the
calibrated grids; they do not establish a generally optimal allocation rule.

\section{Additional Related Work}
\label{app:additional-related-work}
\label{app:related-work}

Beyond the methods discussed in the main text, several studies optimize
sampling through numerical-error objectives. Optimal Stepsize uses dynamic
programming to distill timestep schedules from reference
trajectories~\citep{pei2025optimal}. HSO combines a global search for
initialization with local refinement, using a midpoint-error proxy and a
penalty against excessively close timesteps~\citep{zhu2025hso}.
ART formulates time reparameterization as continuous-time control and learns
the sampling clock using an Euler-error surrogate and actor--critic
updates~\citep{huang2026art}. Concurrent work uses local-error estimates with proportional-integral (PI) step-size control and aggregates per-sample trajectories into reusable static schedules, with evaluations on image generation and CoBit language modeling \citep{kemperman2026adaptive}.

Other approaches optimize a broader set of sampler parameters or adapt the
grid to individual inputs. BNS distills a family of non-stationary solvers,
including their evaluation times, against accurate reference
solutions~\citep{shaul2024bespoke}. DDSS differentiates through sampling and
sample-quality scores to optimize a flexible sampler~\citep{watson2022ddss}.
INDIS learns instance-dependent discretizations for image and video
generation~\citep{yuan2026indis}.

These studies offer different objectives and search spaces for allocating
sampling computation. Our focus is timestep calibration for continuous
DLMs, where a decoder maps terminal representations to token distributions.
\method{} uses both latent error and actual decoder KL to guide calibration,
incorporates language-quality constraints, and reuses one selected schedule
across inputs without changing the model or solver.